\documentclass{article}

\usepackage[final]{neurips_2022}

\usepackage{lscape} 
\usepackage[utf8]{inputenc} 
\usepackage[T1]{fontenc}    
\usepackage{url}            
\usepackage{booktabs}       
\usepackage{multirow}
\usepackage{amsfonts}       
\usepackage{amssymb}
\usepackage{amsmath}
\usepackage{nicefrac}       
\usepackage{microtype}      
\usepackage{adjustbox}

\usepackage{microtype}
\usepackage{graphicx}
\usepackage{booktabs} 
\usepackage{xcolor}         

\usepackage{subcaption}

\usepackage{hyperref}

\newcommand{\dataset}[1]{#1}

\newcommand{\PBkl}{\operatorname{PBkl}}
\newcommand{\PBUB}{\operatorname{PBUB}}

\newcommand{\PBSkl}{\operatorname{PBSkl}}

\newcommand{\TND}{\operatorname{TND}}
\newcommand{\CCTND}{\operatorname{CCTND}}
\newcommand{\CCPBB}{\operatorname{CCPBB}}
\newcommand{\CCPBUB}{\operatorname{CCPBUB}}
\newcommand{\CCPBSkl}{\operatorname{CCPBSkl}}

\newcommand{\Ex}{\operatorname{Ex}}

\usepackage{xifthen}
\usepackage{xparse}
\usepackage{dsfont}
\usepackage{xspace}
\usepackage{xcolor}
\usepackage{wrapfig}
\usepackage{enumitem}

\usepackage{amsthm}
\newtheorem{theorem}{Theorem}

\newtheorem{lemma}[theorem]{Lemma}

\newtheorem{remk}[theorem]{Remark}
\newtheorem{exmp}[theorem]{Example}

\def\FullBox{\hbox{\vrule width 8pt height 8pt depth 0pt}}

\def\qed{\ifmmode\qquad\FullBox\else{\unskip\nobreak\hfil
\penalty50\hskip1em\null\nobreak\hfil\FullBox
\parfillskip=0pt\finalhyphendemerits=0\endgraf}\fi}

\def\qedsketch{\ifmmode\Box\else{\unskip\nobreak\hfil
\penalty50\hskip1em\null\nobreak\hfil$\Box$
\parfillskip=0pt\finalhyphendemerits=0\endgraf}\fi}

\newcommand{\ie} {{\it i.e.,\ }}

\newcommand{\R}{{\mathbb R}} 

\DeclareMathOperator*{\argmin}{arg\,min}
\DeclareMathOperator*{\argmax}{arg\,max}

\newcommand{\lr}[1]{\left (#1\right)} 
\newcommand{\lrs}[1]{\left [#1 \right]} 
\newcommand{\lrc}[1]{\left \{#1\right\}} 

\let\P\undefined
\NewDocumentCommand{\P}{o}{\mathbb P{\IfValueT{#1}{\lr{#1}}}}

\NewDocumentCommand{\E}{o}{\mathbb E\IfValueT{#1}{\lrs{#1}}}

\newcommand{\V}{\mathbb{V}}
\NewDocumentCommand{\Var}{o}{\V\IfValueT{#1}{\lrs{#1}}}

\NewDocumentCommand{\1}{o}{\mathds 1{\IfValueT{#1}{\lr{#1}}}}

\def\Bin{\textsf{Bin}} 

\newcommand{\cH}{\mathcal H}

\newcommand{\cD}{\mathcal{D}}

\newcommand{\cX}{\mathcal{X}}

\newcommand{\cY}{\mathcal{Y}}

\DeclareMathOperator{\KL}{KL}

\DeclareMathOperator{\kl}{kl}

\DeclareMathOperator{\MV}{MV}

\NewDocumentCommand{\CO}{o}{\mathcal O{\IfValueT{#1}{\lr{#1}}}}

\title{Split-$\kl$ and PAC-Bayes-split-$\kl$ Inequalities for Ternary Random Variables}

\author{%
   Yi-Shan Wu\\
   University of Copenhagen\\
   \texttt{yswu@di.ku.dk}\\
   \And
   Yevgeny Seldin \\
   University of Copenhagen \\
   \texttt{seldin@di.ku.dk} \\
}

\begin{document}
\maketitle
\begin{abstract}
    We present a new concentration of measure inequality for sums of independent bounded random variables, which we name a split-$\kl$ inequality. The inequality is particularly well-suited for ternary random variables, which naturally show up in a variety of problems, including analysis of excess losses in classification, analysis of weighted majority votes, and learning with abstention. We demonstrate that for ternary random variables the inequality is simultaneously competitive with the $\kl$ inequality, the Empirical Bernstein inequality, and the Unexpected Bernstein inequality, and in certain regimes outperforms all of them. It resolves an open question by \citet{TS13} and \citet{MGG20} on how to match simultaneously the combinatorial power of the $\kl$ inequality when the distribution happens to be close to binary and the power of Bernstein inequalities to exploit low variance when the probability mass is concentrated on the middle value. We also derive a PAC-Bayes-split-$\kl$ inequality and compare it with the PAC-Bayes-$\kl$, PAC-Bayes-Empirical-Bennett, and PAC-Bayes-Unexpected-Bernstein inequalities in an analysis of excess losses and in an analysis of a weighted majority vote for several UCI datasets. Last but not least, our study provides the first direct comparison of the Empirical Bernstein and Unexpected Bernstein inequalities and their PAC-Bayes extensions.
    
\end{abstract}

\section{Introduction}\label{sec:Intro}

Concentration of measure inequalities for sums of independent random variables are the most fundamental analysis tools in statistics and many other domains \citep{BLM13}. Their history stretches almost a century back, and inequalities such as Hoeffding's \citep{Hoe63} and Bernstein's \citep{Ber46} are the main work horses of learning theory. 

For binary random variables, one of the tightest concentration of measure inequalities is the $\kl$ inequality \citep{Mau04,Lan05,FBBT21,FBB22}, which is based on combinatorial properties of a sum of $n$ independent random variables.\footnote{The Binomial tail bound is slightly tighter, but it does not extend to the PAC-Bayes setting \citep{Lan05}. Our split-$\kl$ approach can be directly applied to obtain a ``split-Binomial-tail'' inequality.} However, while being extremely tight for binary random variables and applicable to any bounded random variables, the $\kl$ inequality is not necessarily a good choice for sums of bounded random variables that can take more than two values. In the latter case, the Empirical Bernstein \citep{MSA08,AMS09,MP09} and the Unexpected Bernstein \citep{CBMS07,MGG20} inequalities can be significantly tighter due to their ability to exploit low variance, as shown by \citet{TS13}. However, the Empirical and Unexpected Bernstein inequalities are loose for binary random variables \citep{TS13}.

The challenge of exploiting low variance and, at the same time, matching the tightness of the $\kl$ inequality if a distribution happens to be close to binary, was faced by multiple prior works \citep{TS13,MGG20,WMLIS21}, but remained an open question. We resolve this question for the case of ternary random variables. Such random variables appear in a variety of applications, and we illustrate two of them. One is a study of excess losses, which are differences between the zero-one losses of a prediction rule $h$ and a reference prediction rule $h^*$, $Z = \ell(h(X),Y) - \ell(h^*(X),Y) \in \{-1,0,1\}$. \citet{MGG20} have applied the PAC-Bayes-Unexpected-Bernstein bound to excess losses in order to improve generalization bounds for classification. Another example of ternary random variables is the tandem loss with an offset, defined by $\ell_\alpha(h(X),h'(X),Y) = (\ell(h(X),Y) - \alpha)(\ell(h'(X),Y) - \alpha) \in \lrc{\alpha^2, -\alpha(1-\alpha),(1-\alpha)^2}$. \citet{WMLIS21} have applied the PAC-Bayes-Empirical-Bennett inequality to the tandem loss with an offset to obtain a generalization bound for the weighted majority vote. Yet another potential application, which we leave for future work, is learning with abstention \citep{CDG+18,TBB+19}.

We present the split-$\kl$ inequality, which simultaneously matches the tightness of the Empirical/Unexpected Bernstein and the $\kl$, and outperforms both for certain distributions. It works for sums of any bounded random variables $Z_1,\dots,Z_n$, not only the ternary ones, but it is best suited for ternary random variables, for which it is almost tight (in the same sense, as the $\kl$ is tight for binary random variables). The idea behind the split-$\kl$ inequality is to write a random variable $Z$ as  $Z = \mu + Z^+ - Z^-$, where $\mu$ is a constant, $Z^+ = \max\{0, Z-\mu\}$, and $Z^-=\max\{0,\mu-Z\}$. Then $\E[Z] = \mu + \E[Z^+] - \E[Z^-]$ and, given an i.i.d.\ sample $Z_1,\dots,Z_n$, we can bound the distance between $\frac{1}{n}\sum_{i=1}^n Z_i$ and $\E[Z]$ by using $\kl$ upper and lower bounds on the distances between $\frac{1}{n}\sum_{i=1}^n Z_i^+$ and $\E[Z^+]$, and $\frac{1}{n}\sum_{i=1}^n Z_i^-$ and $\E[Z^-]$, respectively. For ternary random variables $Z \in \lrc{a,b,c}$ with $a \leq b \leq c$, the best split is to take $\mu=b$, then both $Z^+$ and $Z^-$ are binary and the $\kl$ upper and lower bounds for their rescaled versions are tight and, therefore, the split-$\kl$ inequality for $Z$ is also tight. Thus, this approach provides the best of both worlds: the combinatorial tightness of the $\kl$ bound and exploitation of low variance when the probability mass on the middle value happens to be large, as in Empirical Bernstein inequalities. We further elevate the idea to the PAC-Bayes domain and derive a PAC-Bayes-split-$\kl$ inequality.

We present an extensive set of experiments, where we first compare the $\kl$, Empirical Bernstein, Unexpected Bernstein, and split-$\kl$ inequalities applied to (individual) sums of independent random variables in simulated data, and then compare the PAC-Bayes-$\kl$, PAC-Bayes-Unexpected-Bersnstein, PAC-Bayes-split-$\kl$, and, in some of the setups, PAC-Bayes-Empirical-Bennett, for several prediction models on several UCI datasets. In particular, we evaluate the bounds in the linear classification setup studied by \citet{MGG20} and in the weighted majority prediction setup studied by \citet{WMLIS21}. To the best of our knowledge, this is also the first time when the Empirical Bernstein and the Unexpected Bernstein inequalities are directly compared, with and without the PAC-Bayesian extension. In Appendix~\ref{sec:comp:Unexpected-Bernstein} we also show that an inequality introduced by \citet{CBMS07} yields a relaxation of the Unexpected Bernstein inequality by \citet{MGG20}.
\section{Concentration of Measure Inequalities for Sums of Independent Random Variables}

We start with the most basic question in probability theory and statistics: how far can an average of an i.i.d.\ sample $Z_1,\dots,Z_n$ deviate from its expectation? We cite the major existing inequalities, the $\kl$, Empirical Bernstein, and Unexpected Bernstein, then derive the new split-$\kl$ inequality, and then provide a numerical comparison.

\subsection{Background}

We use $\KL(\rho\|\pi)$ to denote the Kullback-Leibler divergence between two probability distributions, $\rho$ and $\pi$ \citep{CT06}. We further use $\kl(p\|q)$ as a shorthand for the Kullback-Leibler divergence between two Bernoulli distributions with biases $p$ and $q$, namely $\kl(p\|q)=\KL((1-p,p)\|(1-q,q))$.
For $\hat{p}\in[0,1]$ and $\varepsilon\geq 0$ we define the upper and lower inverse of $\kl$, respectively, as $\kl^{-1,+}(\hat{p},\varepsilon):=\max\lrc{p: p\in[0,1] \text{ and } \kl(\hat{p}\|p)\leq \varepsilon}$ and $\kl^{-1,-}(\hat{p},\varepsilon):=\min\lrc{p: p \in [0,1] \text{ and } \kl(\hat{p}\|p)\leq \varepsilon}$. 

The first inequality that we cite is the $\kl$ inequality.
\begin{theorem}[$\kl$ Inequality~\citep{Lan05,FBBT21,FBB22}]\label{thm:kl}
Let $Z_1,\cdots,Z_n$ be i.i.d.\ random variables bounded in the $[0,1]$ interval and with $\E[Z_i] = p$ for all $i$. Let $\hat{p}=\frac{1}{n}\sum_{i=1}^n Z_i$ be their empirical mean. Then, for any $\delta\in(0,1)$
:
\[
\P[\kl(\hat p\|p)\geq \frac{\ln\frac{1}{\delta}}{n}] \leq \delta
\]
and, by inversion of the $\kl$,
\begin{equation}\label{eq:kl+}
\P[p\geq \kl^{-1,+}\lr{\hat p, \frac{1}{n}\ln\frac{1}{\delta}}
] \leq \delta,
\end{equation}
\begin{equation}\label{eq:kl-}
\P[p \leq \kl^{-1,-}\lr{\hat p, \frac{1}{n}\ln\frac{1}{\delta}}]\leq \delta.
\end{equation}
\end{theorem}
We note that the PAC-Bayes-$\kl$ inequality (Theorem~\ref{thm:pbkl} below) is based on the inequality $\E[e^{n\kl(\hat p\|p)}]\leq 2\sqrt{n}$ \citep{Mau04}, which gives $\P[\kl(\hat p\|p)\geq\frac{\ln \frac{2\sqrt{n}}{\delta}}{n}]\leq \delta$. \citet{FBBT21,FBB22} reduce the logarithmic factor down to $\ln\frac{1}{\delta}$ by basing the proof on Chernoff's inequality, but this proof technique cannot be combined with PAC-Bayes. Therefore, when we move on to PAC-Bayes we pay the extra $\ln 2\sqrt{n}$ factor in the bounds. It is a long-standing open question whether this factor can be reduced in the PAC-Bayesian setting \citep{FBBT21}.

Next we cite two versions of the Empirical Bernstein inequality.

\begin{theorem}[Empirical Bernstein Inequality~\citep{MP09}]\label{thm:EBernstein}
Let $Z_1,\cdots,Z_n$ be i.i.d.\ random variables bounded in a $[a,b]$ interval for some $a,b\in\R$, and with $\E[Z_i]=p$ for all $i$. Let $\hat{p}=\frac{1}{n}\sum_{i=1}^n Z_i$ be the empirical mean and let $\hat{\sigma}=\frac{1}{n-1}\sum\limits_{i=1}^n(Z_i-\hat p)^2$ be the empirical variance. Then for any $\delta\in(0,1)$
:
\begin{equation}\label{eq:EBernstein}
\P[p\geq \hat{p} + \sqrt{\frac{2\hat{\sigma}\ln\frac{2}{\delta}}{n}} + \frac{7(b-a)\ln\frac{2}{\delta}}{3(n-1)}]\leq \delta.
\end{equation}
\end{theorem}

\begin{theorem}[Unexpected Bernstein Inequality~\citep{FGL15,MGG20}]\label{thm:Unexpected-Bernstein}
Let $Z_1,\cdots,Z_n$ be i.i.d.\ random variables bounded from above by $b$ for some $b>0$, and with $\E[Z_i]=p$ for all $i$. Let $\hat{p}=\frac{1}{n}\sum_{i=1}^n Z_i$ be the empirical mean and let $\hat\sigma=\frac{1}{n}\sum_{i=1}^n Z_i^2$ be the empirical mean of the second moments. Let $\psi(u):=u-\ln(1+u)$ for $u > -1$. Then, for any $\gamma\in(0,1/b)$ and any $\delta\in(0,1)$
:
\begin{equation}\label{eq:Unexpected-Bernstein}
\P[p\geq \hat{p} + \frac{\psi(-\gamma b)}{\gamma b^2}\hat{\sigma} + \frac{\ln\frac{1}{\delta}}{\gamma n}]\leq \delta.
\end{equation}
\end{theorem}
To facilitate a comparison with other bounds, Theorem~\ref{thm:Unexpected-Bernstein} provides a slightly different form of the Unexpected Bernstein inequality than the one used by \citet{MGG20}. We provide a proof of the theorem in Appendix~\ref{sec:pf_thm:Unexpected-Bernstein}, which is based on the Unexpected Bernstein Lemma~\citep{FGL15}. We note that an inequality proposed by \citet{CBMS07} can be used to derive a relaxed version of the Unexpected Bernstein inequality, as discussed in Appendix~\ref{sec:comp:Unexpected-Bernstein}.  

\subsection{The Split-kl Inequality}

Let $Z$ be a random variable bounded in a $[a,b]$ interval for some $a,b\in\R$ and let $\mu\in[a,b]$ be a constant. We decompose $Z = \mu + Z^+ - Z^-$, where $Z^+=\max(0,Z-\mu)$ and $Z^- = \max(0, \mu-Z)$. Let $p=\E[Z]$, $p^+=\E[Z^+]$, and $p^-=\E[Z^-]$. For an i.i.d.\ sample $Z_1,\dots,Z_n$ let $\hat p^+=\frac{1}{n}\sum_{i=1}^n Z_i^+$ and $\hat p^-=\frac{1}{n}\sum_{i=1}^n Z_i^-$.

With these definitions we present the split-$\kl$ inequality.

\begin{theorem}[Split-$\kl$ inequality]\label{thm:Split_kl}
Let $Z_1,\dots,Z_n$ be i.i.d.\ random variables in a $[a,b]$ interval for some $a,b\in\R$, then for any $\mu\in [a,b]$ and $\delta\in(0,1)$:
\begin{equation}\label{eq:Split-kl}
\P[p\geq \mu + (b-\mu)\kl^{-1,+}\lr{\frac{\hat p^+}{b-\mu},\frac{1}{n}\ln\frac{2}{\delta}} - (\mu-a)\kl^{-1,-}\lr{\frac{\hat p^-}{\mu-a},\frac{1}{n}\ln\frac{2}{\delta}}]\leq \delta.
\end{equation}
\end{theorem}
\begin{proof}
\begin{align*}
&\P[p\geq \mu + (b-\mu)\kl^{-1,+}\lr{\frac{\hat p^+}{b-\mu},\frac{1}{n}\ln\frac{2}{\delta}} - (\mu-a)\kl^{-1,-}\lr{\frac{\hat p^-}{\mu-a},\frac{1}{n}\ln\frac{2}{\delta}}]\\
&\qquad\leq \P[p^+ \geq (b-\mu)\kl^{-1,+}\lr{\frac{\hat p^+}{b-\mu},\frac{1}{n}\ln\frac{2}{\delta}}] + \P[p^-\leq (\mu-a)\kl^{-1,-}\lr{\frac{\hat p^-}{\mu-a},\frac{1}{n}\ln\frac{2}{\delta}}]\\
&\qquad\leq \delta,
\end{align*}
where the last inequality follows by application of the $\kl$ upper and lower bounds from Theorem~\ref{thm:kl} to the first and second terms in the middle line, respectively.
\end{proof}

For ternary random variables the best choice is to take $\mu$ to be the middle value, then the resulting $Z^+$ and $Z^-$ are binary and the corresponding $\kl$ upper and lower bounds on $p^+$ and $p^-$ are tight, and the resulting split-$\kl$ bound is tight. The inequality can be applied to any bounded random variables, but same way as the $\kl$ inequality is not necessarily a good choice for bounded random variables, if the distribution is not binary, the split-$\kl$ in not necessarily a good choice if the distribution is not ternary.

\subsection{Empirical Comparison}\label{sec:empirical_comparison}

We present an empirical comparison of the tightness of the above four concentration inequalities: the kl, the Empirical Bernstein, the Unexpected Bernstein, and the split-$\kl$. We take $n$ i.i.d.\ samples $Z_1,\dots,Z_n$ taking values in $\{-1,0,1\}$. The choice is motivated both by instructiveness of presentation and by subsequent applications to excess losses. We let $p_{-1} = \P[Z=-1]$, $p_0=\P[Z=0]$, and $p_1 = \P[Z=1]$, where $p_{-1}+p_0+p_1 = 1$. Then $p = \E[Z] = p_1-p_{-1}$. We also let $\hat p = \frac{1}{n}\sum_{i=1}^n Z_i$. 

In Figure~\ref{fig:main:plain_plots} we plot the difference between the bounds on $p$ given by the inequalities \eqref{eq:kl+}, \eqref{eq:EBernstein}, \eqref{eq:Unexpected-Bernstein}, and \eqref{eq:Split-kl}, and $\hat p$. Lower values in the plot correspond to tighter bounds. To compute the $\kl$ bound we first rescale the losses to the $[0,1]$ interval, and then rescale the bound back to the $[-1,1]$ interval. For the Empirical Bernstein bound we take $a=-1$ and $b=1$. For the Unexpected Bernstein bound we take a grid of $\gamma \in \{1/(2b), \cdots, 1/(2^k b)\}$ for $k=\lceil \log_2(\sqrt{n/\ln(1/\delta)}/2) \rceil$ and a union bound over the grid, as proposed by \citet{MGG20}. For the split-$\kl$ bound we take $\mu$ to be the middle value, 0, of the ternary random variable. In the experiments we take $\delta = 0.05$, and truncate the bounds at $1$.

In the first experiment, presented in Figure~\ref{fig:sim:n100_rate0.5}, we take $p_{-1} = p_1 = (1-p_0)/2$ and plot the difference between the values of the bounds and $\hat p$ as a function of $p_0$. For $p_0 = 0$ the random variable $Z$ is Bernoulli and, as expected, the $\kl$ inequality performs the best, followed by split-$\kl$, and then Unexpected Bernstein. As $p_0$ grows closer to 1, the variance of $Z$ decreases and, also as expected, the $\kl$ inequality falls behind, whereas split-$\kl$ and Unexpected Bernstein go closely together. Empirical Bernstein falls behind all other bounds throughout most of the range, except slightly outperforming $\kl$ when $p_0$ gets very close to 1. 

In the second experiment, presented in Figure~\ref{fig:sim:n100_rate0.99}, we take a skewed random variable with $p_1=0.99(1-p_0)$ and $p_{-1}=0.01(1-p_0)$, and again plot the difference between the values of the bounds and $\hat p$ as a function of $p_0$. This time the $\kl$ also starts well for $p_0$ close to zero, but then falls behind due to its inability of properly handling the values inside the interval. Unexpected Bernstein exhibits the opposite trend due to being based on uncentered second moment, which is high when $p_0$ is close to zero, even though the variance is small in this case. Empirical Bernstein lags behind all other bounds for most of the range due to poor constants, whereas split-$\kl$ matches the tightest bounds, the $\kl$ and Unexpected Bernstein, at the endpoints of the range of $p_0$, and outperforms all other bounds in the middle of the range, around $p_0=0.6$, due to being able to exploit the combinatorics of the problem.

The experiments demonstrate that for ternary random variables the split-$\kl$ is a powerful alternative to existing concentration of measure inequalities. To the best of our knowledge, this is also the first empirical evaluation of the Unexpected Bernstein inequality, and it shows that in many cases it is also a powerful inequality. We also observe that in most settings the Empirical Bernstein is weaker than the other three inequalities we consider. Numerical evaluations in additional settings are provided in Appendix~\ref{app:sec:empirical-comparison}.
\begin{figure}[t]
	\begin{subfigure}[b]{.45\textwidth}
		\includegraphics[width=\textwidth]{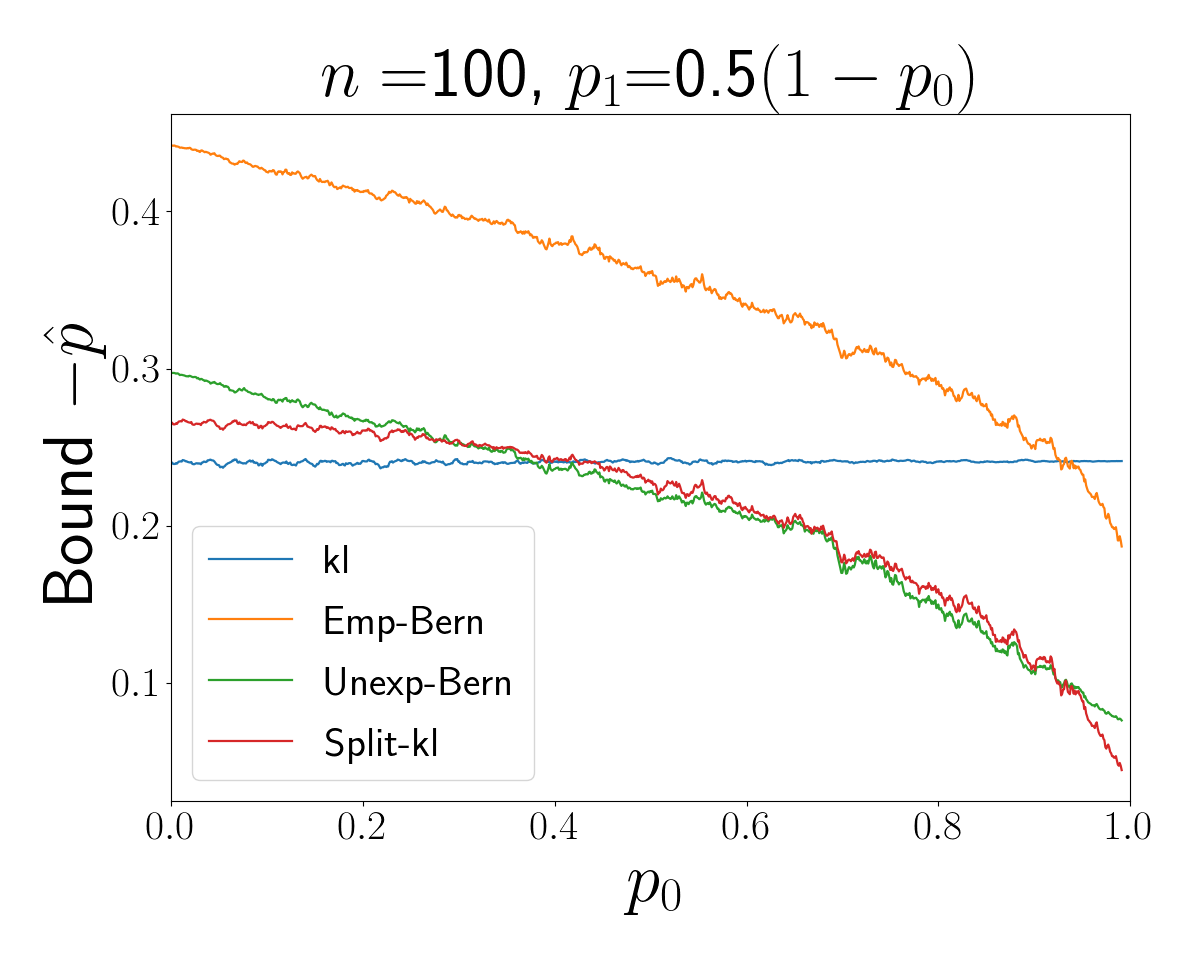}
		\caption{Comparison of the concentration bounds with $n=100$, $\delta=0.05$ and  $p_{-1}=p_1=0.5(1-p_0)$.}
		\label{fig:sim:n100_rate0.5}
	\end{subfigure}
	\hfill
	\begin{subfigure}[b]{.45\textwidth}
		\includegraphics[width=\textwidth]{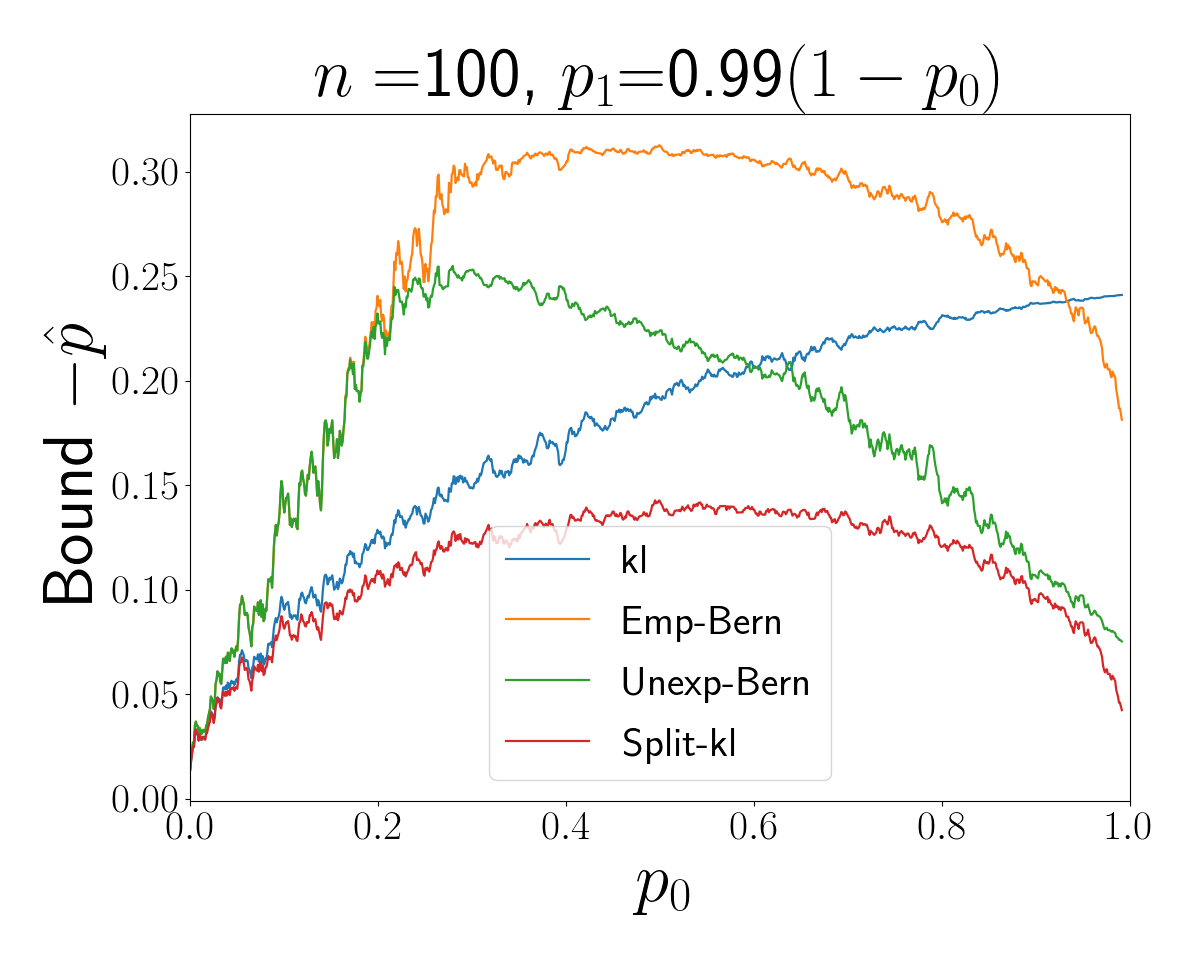}
		\caption{Comparison of the concentration bounds with $n=100$, $\delta=0.05$, $p_1=0.99(1-p_0)$, and $p_{-1}=0.01(1-p_0)$.}
		\label{fig:sim:n100_rate0.99}
	\end{subfigure}
	\caption{Empirical comparison of the concentration bounds. }
	\label{fig:main:plain_plots}
\end{figure}

\section{PAC-Bayesian Inequalities}

Now we elevate the basic concentration of measure inequalities to the PAC-Bayesian domain. We start with the supervised learning problem setup, then provide a background on existing PAC-Bayesian inequalities, and finish with presentation of the PAC-Bayes-split-$\kl$ inequality.

\subsection{Supervised Learning Problem Setup and Notations}
Let $\cX$ be a sample space, $\cY$ be a label space, and let $S=\lrc{(X_i,Y_i)}_{i=1}^n$ be an i.i.d.\  sample drawn according to an unknown distribution $\cD$ on the product-space $\cX\times \cY$.
Let $\cH$ be a hypothesis space containing hypotheses $h:\cX\rightarrow \cY$. The quality of a hypothesis $h$ is measured using the zero-one loss $\ell(h(X),Y)=\1[h(X)\neq Y]$, where $\1[\cdot]$ is the indicator function.
The expected loss of $h$ is denoted by $L(h)=\E_{(X,Y)\sim \cD}\lrs{\ell(h(X),Y)}$, and the empirical loss of $h$ on a sample $S$ is denoted by $\hat{L}(h,S)=\frac{1}{|S|}\sum_{(X,Y)\in S}\ell(h(X),Y)$. We use $\E_\cD[\cdot]$ as a shorthand for $\E_{(X,Y)\sim \cD}[\cdot]$.

PAC-Bayesian bounds bound the generalization error of Gibbs prediction rules. For each input $X\in\cX$, Gibbs prediction rule associated with a distribution $\rho$ on $\cH$ randomly draws a hypothesis $h\in\cH$ according to $\rho$ and predicts $h(X)$. The expected loss of the Gibbs prediction rule is $\E_{h\sim\rho}[L(h)]$ and the empirical loss is $\E_{h\sim\rho}[\hat{L}(h,S)]$. We use $\E_\rho[\cdot]$ as a shorthand for $\E_{h\sim\rho}[\cdot]$.

\subsection{PAC-Bayesian Analysis Background}

Now we present a brief background on the relevant results from the PAC-Bayesian analysis.

\paragraph{PAC-Bayes-$\kl$ Inequality}
The PAC-Bayes-$\kl$ inequality cited below is one of the tightest known generalization bounds on the expected loss of the Gibbs prediction rule.
\begin{theorem}[PAC-Bayes-$\kl$ Inequality, \citealp{See02}, \citealp{Mau04}] \label{thm:pbkl}
For any probability distribution $\pi$ on ${\cal H}$ that is independent of $S$ and any $\delta \in (0,1)$:
\begin{equation}
\label{eq:pbkl}
\P[\exists \rho\in{\mathcal{P}}: \kl\lr{\E_\rho[\hat L(h,S)]\middle\|\E_\rho\lrs{L(h)}} \geq \frac{\KL(\rho\|\pi) + \ln(2 \sqrt n/\delta)}{n}]\leq \delta,
\end{equation}
where $\mathcal{P}$ is the set of all possible probability distributions on $\cH$ that can depend on $S$.
\end{theorem}

The following relaxation of the PAC-Bayes-$\kl$ inequality based on Refined Pinsker's relaxation of the $\kl$ divergence helps getting some intuition about the bound \citep{McA03}. With probability at least $1-\delta$, for all $\rho\in\mathcal{P}$ we have
\begin{equation}\label{eq:pbkl-relaxed}
\E_\rho[L(h)] \leq \E_\rho[\hat{L}(h,S)] + \sqrt{2\E_\rho[\hat{L}(h,S)]\frac{\KL(\rho\|\pi) + \ln(2 \sqrt n/\delta)}{n}} + \frac{2\lr{\KL(\rho\|\pi) + \ln(2 \sqrt n/\delta)}}{n}.
\end{equation}
If $\E_\rho[\hat{L}(h,S)]$ is close to zero, the middle term in the inequality above vanishes, leading to so-called "fast convergence rates" (convergence of $\E_\rho[\hat{L}(h,S)]$ to $\E_\rho[L(h)]$ at the rate of $1/n$). However, achieving low $\E_\rho[\hat{L}(h,S)]$ is not always possible \citep{DR17,ZVAAO19}. Subsequent research in PAC-Bayesian analysis has focused on two goals: (1) achieving fast convergence rates when the variance of prediction errors is low (and not necessarily the errors themselves), and (2) reducing the $\KL(\rho\|\pi)$ term, which may be quite large for large hypothesis spaces. For the first goal \citet{TS13} developed the PAC-Bayes-Empirical-Bernstein inequality and \citet{MGG20} proposed to use excess losses and also derived the alternative PAC-Bayes-Unexpected-Bernstein inequality. For the second goal \citet{APS07} suggested to use informed priors and \citet{MGG20} perfected the idea by proposing to average over "forward" and "backward" construction with informed prior. Next we explain the ideas behind the excess losses and informed priors in more details.  


\paragraph{Excess Losses}
Let $h^*$ be a reference prediction rule that is independent of $S$. 
%
We define the excess loss of a prediction rule $h$ with respect to the reference $h^*$ by
\[
\Delta_\ell(h(X),h^*(X),Y) = \ell(h(X),Y) - \ell(h^*(X),Y).
\] 
If $\ell$ is the zero-one loss, the excess loss naturally gives rise to ternary random variables, but it is well-defined for any real-valued loss function. We use $\Delta_L(h,h^*)=\E_D[\Delta_\ell(h(X),h^*(X),Y)]
=L(h)-L(h^*)$ to denote the expected excess loss of $h$ relative to $h^*$ and $\Delta_{\hat{L}}(h,h',S)=\frac{1}{|S|}\sum_{(X,Y)\in S} \Delta_\ell(h(X),h^*(X),Y) = \hat L(h) - \hat L(h^*)$ to denote the empirical excess loss of $h$ relative to $h^*$. The expected loss of a Gibbs prediction rule can then be written as
\[
\E_\rho[L(h)] = \E_\rho[\Delta_L(h,h^*)] + L(h^*).
\]
A bound on $\E_\rho[L(h)]$ can thus be decomposed into a summation of a PAC-Bayes bound on $\E_\rho[\Delta_L(h,h^*)]$ and a bound on $L(h^*)$. When the variance of the excess loss is small, we can use tools that exploit small variance, such as the  PAC-Bayes-Empirical-Bernstein, PAC-Bayes-Unexpected-Bernstein, or PAC-Bayes-Split-$\kl$ inequalities proposed below, to achieve fast convergence rates for the excess loss. Bounding $L(h^*)$ involves just a single prediction rule and does not depend on the value of $\KL(\rho\|\pi)$. We note that it is essential that the variance and not just the magnitude of the excess loss is small. For example, if the excess losses primarily take values in $\{-1,1\}$ and average out to zero, fast convergence rates are impossible. 

\paragraph{Informed Priors}

The idea behind informed priors is to split the data into two subsets, $S=S_1\cup S_2$, and to use $S_1$ to learn a prior $\pi_{S_1}$, and then use it to learn a posterior on $S_2$ \citet{APS07}. Note that since the size of $S_2$ is smaller than the size of $S$, this approach gains in having potentially smaller $\KL(\rho\|\pi_{S_1})$, but loses in having a smaller sample size in the denominator of the PAC-Bayes bounds. The balance between the advantage and disadvantage depends on the data: for some data sets it strengthens the bounds, but for some it weakens them. \citet{MGG20} perfected the approach by proposing to use it in the "forward" and "backward" direction and average over the two. Let $S_1$ and $S_2$ be of equal size. The "forward" part uses $S_1$ to train $\pi_{S_1}$ and then computes a posterior on $S_2$, while the "backward" part uses $S_2$ to train $\pi_{S_2}$ and then computes a posterior on $S_1$. Finally, the two posteriors are averaged with equal weight and the $\KL$ term becomes $\frac{1}{2}\lr{\KL(\rho\|\pi_{S_1})+\KL(\rho\|\pi_{S_2})}$. See \citep{MGG20} for the derivation.

\paragraph{Excess Losses and Informed Priors} Excess losses and informed priors make an ideal combination. If we split $S$ into two equal parts, $S=S_1\cup S_2$, we can use $S_1$ to train both a reference prediction rule $h_{S_1}$ and a prior $\pi_{S_1}$, and then learn a PAC-Bayes posterior on $S_2$, and the other way around. By combining the "forward" and "backward" approaches we can write
%
\begin{equation}\label{eq:excess&inform}
    \E_\rho[L(h)] = \frac{1}{2}\E_\rho[\Delta_L(h,h_{S_1})] + \frac{1}{2}\E_\rho[\Delta_L(h,h_{S_2})] + \frac{1}{2}\lr{L(h_{S_1})+L(h_{S_2})},
\end{equation}
and we can use PAC-Bayes to bound the first term using the prior $\pi_{S_1}$ and the data in $S_2$, and to bound the second term using the prior $\pi_{S_2}$ and the data in $S_1$, and we can bound $L(h_{S_1})$ and $L(h_{S_2})$ using the "complementary" data in $S_2$ and $S_1$, respectively.


\paragraph{PAC-Bayes-Empirical-Bernstein Inequalities}
The excess losses are ternary random variables taking values in $\{-1,0,1\}$ and, as we have already discussed, the $\kl$ inequality is not well-suited for them. PAC-Bayesian inequalities tailored for non-binary random variables were derived by \citet{SLCB+12}, \citet{TS13}, \citet{WMLIS21}, and \citet{MGG20}. \citet{SLCB+12} derived the PAC-Bayes-Bernstein oracle bound, which assumes knowledge of the variance. \citet{TS13} made it into an empirical bound by deriving the PAC-Bayes-Empirical-Bernstein bound for the variance and plugging it into the PAC-Bayes-Bernstein bound of \citeauthor{SLCB+12}. \citet{WMLIS21} derived an oracle PAC-Bayes-Bennett inequality, which again assumes oracle knowledge of the variance, and showed that it is always at least as tight as the PAC-Bayes-Bernstein, and then also plugged in the PAC-Bayes-Empirical-Bernstein bound on the variance. \citet{MGG20} derived the PAC-Bayes-Unexpected-Bernstein inequality, which directly uses the empirical second moment. Since we have already shown that the Unexpected Bernstein inequality is tighter than the Empirical Bernstein, and since the approach of \citeauthor{WMLIS21} requires a combination of two inequalities, PAC-Bayes-Empirical-Bernstein for the variance and PAC-Bayes-Bennett for the loss, whereas the approach of \citeauthor{MGG20} only makes a single application of PAC-Bayes-Unexpected-Bernstein, we only compare our work to the latter. 

We cite the inequality of \citet{MGG20}, which applies to an arbitrary loss function. 
We use $\tilde \ell$ and matching tilde-marked quantities to distinguish it from the zero-one loss $\ell$. For any $h\in\cH$, let $\tilde L(h)=\E_D[\tilde \ell(h(X),Y)]$ be the expected tilde-loss of $h$, 
and let $\hat {\tilde L}(h,S) = \frac{1}{|S|}\sum_{(X,Y)\in S} \tilde \ell(h(X),Y)$ be the empirical tilde-loss of $h$ on a sample $S$.

\begin{theorem}[PAC-Bayes-Unexpected-Bernstein inequality~\citep{MGG20}]\label{thm:PB-unexpected-bernstein}
Let $\tilde \ell(\cdot,\cdot)$ be an arbitrary loss function bounded from above by $b$ for some $b>0$, and assume that $\hat{\tilde \V}(h,S)=\frac{1}{|S|}\sum_{(X,Y)\in S}\tilde \ell(h(X),Y)^2$ is finite for all $h$. Let  $\psi(u):=u-\ln(1+u)$ for $u>-1$. Then for any distribution $\pi$ on $\cH$ that is independent of $S$, any $\gamma \in(0,1/b)$, and any $\delta\in(0,1)$
:
\[
\P[\exists \rho\in\mathcal{P}: \E_\rho[\tilde L(h)] \geq \E_\rho[\hat{\tilde L}(h,S)] + \frac{\psi(-\gamma b)}{\gamma b^2}\E_\rho[\hat{\tilde \Var}(h,S)] + \frac{\KL(\rho\|\pi)+\ln \frac{1}{\delta}}{\gamma n}]\leq\delta,
\]
where $\mathcal{P}$ is the set of all possible probability  distributions on $\cH$ that can depend on $S$.
\end{theorem}

In optimization of the bound, we take the same grid of $\gamma \in \{1/(2b), \cdots, 1/(2^k b)\}$ for $k=\lceil \log_2(\sqrt{n/\ln(1/\delta)}/2) \rceil$ and a union bound over the grid, as we did for Theorem~\ref{thm:Unexpected-Bernstein}.


\subsection{PAC-Bayes-Split-$\kl$ Inequality}\label{sec:PBSkl}

Now we present our PAC-Bayes-Split-$\kl$ inequality. 
For an arbitrary loss function $\tilde{\ell}$ taking values in a $[a,b]$ interval for some $a,b\in\R$, let $\tilde{\ell}^+:=\max\{0, \tilde{\ell}-\mu\}$ and $\tilde{\ell}^-:=\max\{0,\mu-\tilde{\ell}\}$ for some $\mu\in[a,b]$. For any $h\in\cH$, let $\tilde{L}^+(h)=\E_D[\tilde{\ell}^+(h(X), Y)]$ and $\tilde{L}^-(h)=\E_D[\tilde{\ell}^-(h(X), Y)]$. The corresponding empirical losses are denoted by $\hat{\tilde{L}}^+(h,S)=\frac{1}{n}\sum_{i=1}^n \tilde{\ell}^+(h(X_i),Y_i)$ and $\hat{\tilde{L}}^-(h,S)=\frac{1}{n}\sum_{i=1}^n \tilde{\ell}^-(h(X_i),Y_i)$. 

\begin{theorem}[PAC-Bayes-Split-kl Inequality]\label{thm:pac-bayes-split-kl-inequality}
Let $\tilde{\ell}(\cdot,\cdot)$ be an arbitrary loss function taking values in a $[a,b]$ interval for some $a,b\in\R$. Then for any distribution $\pi$ on $\cH$ that is independent of $S$, any $\mu\in[a,b]$, and any $\delta\in(0,1)$
:
\begin{align*}
    \P\Bigg[\exists\rho\in\mathcal{P}:\E_\rho[\tilde{L}(h)]&\geq \mu + (b-\mu)\kl^{-1,+}\lr{\frac{\E_\rho[\hat{\tilde{L}}^+(h,S)]}{b-\mu},\frac{\KL(\rho\|\pi)+\ln\frac{4\sqrt{n}}{\delta}}{n}}\\
    &\qquad\ - (\mu-a)\kl^{-1,-}\lr{\frac{\E_\rho[\hat{\tilde{L}}^-(h,S)]}{\mu-a},\frac{\KL(\rho\|\pi)+\ln\frac{4\sqrt{n}}{\delta}}{n}}\Bigg]\leq \delta,
\end{align*}
where $\mathcal{P}$ is the set of all possible probability distributions on $\cH$ that can depend on $S$.
\end{theorem}
\begin{proof}
We have $\E_\rho[\tilde{L}(h)]=\mu+\E_\rho[\tilde{L}^+(h)]-\E_\rho[\tilde{L}^-(h)]$. Similar to the proof of Theorem~\ref{thm:Split_kl}, we take a union bound of PAC-Bayes-$\kl$ upper bound on $\E_\rho[\tilde{L}^+(h)]$ and PAC-Bayes-$\kl$ lower bound on $\E_\rho[\tilde{L}^-(h)]$.
\end{proof}

\subsection{PAC-Bayes-split-$\kl$ with Excess Loss and Informed Prior}

Looking back at the expected loss decomposition in equation \eqref{eq:excess&inform}, we can use PAC-Bayes-split-$\kl$ to bound the first two terms and a bound on the binomial tail distribution to bound the last term.
For $n$ i.i.d.\ Bernoulli random variables $Z_1,\dots, Z_n$ with bias $p\in(0,1)$, we define the binomial tail distribution $\Bin(n,k,p)=\P[\sum_{i=1}^n X_i\leq k]$ and its inverse $\Bin^{-1}(n,k,\delta)=\max\lrc{p: p\in [0,1] \text{ and } \Bin(n,k,p)\geq \delta}$. The following theorem relates $\hat p = \frac{1}{n}\sum_{i=1}^n Z_i$ and $p$.
\begin{theorem}[Test Set Bound~\citep{Lan05}]\label{thm:test-set-bound}
Let $Z_1,\dots, Z_n$ be $n$ i.i.d.\ Bernoulli random variables with bias $p\in(0,1)$ and let $\hat{p}=\frac{1}{n}\sum_{i=1}^n Z_i$ be the empirical mean. Then for any $\delta\in(0,1)$:
\[
\P[p\geq \Bin^{-1}(n, n\hat{p}, \delta)]\leq \delta.
\]
\end{theorem}

By applying Theorems~\ref{thm:pac-bayes-split-kl-inequality} and \ref{thm:test-set-bound} to equation \eqref{eq:excess&inform} we obtain the following result.
\begin{theorem}\label{thm:together_refined}
For any $\mu\in[-1,1]$ and any $\delta\in(0,1)$:
\[
    \P[\exists \rho\in\mathcal{P}: \E_\rho[L(h)] \geq \mu + (1-\mu)(a) - (\mu+1)(b) + \frac{1}{2}(c)]\leq \delta,
\]
where $\mathcal{P}$ is the set of all possible probability distributions on $\cH$ that can depend on $S$,
\begin{equation*}
    (a)= \kl^{-1,+}\lr{\frac{1}{2} \frac{\E_{\rho}[\Delta_{\hat{L}}^+(h,h_{S_1},S_2)]}{1-\mu} + \frac{1}{2}\frac{\E_{\rho}[\Delta_{\hat{L}}^+(h,h_{S_2},S_1)]}{1-\mu}, \frac{\KL(\rho\|\pi) + \ln\frac{8\sqrt{n/2}}{\delta}}{n/2} },
\end{equation*}
\begin{equation*}
    (b)= \kl^{-1,-}\lr{\frac{1}{2} \frac{\E_{\rho}[\Delta_{\hat{L}}^-(h,h_{S_1},S_2)]}{\mu+1} + \frac{1}{2}\frac{\E_{\rho}[\Delta_{\hat{L}}^-(h,h_{S_2},S_1)]}{\mu+1}, \frac{\KL(\rho\|\pi) + \ln\frac{8\sqrt{n/2}}{\delta}}{n/2} },
\end{equation*}
in which $\pi=\frac{1}{2}\pi_{S_1}+\frac{1}{2}\pi_{S_2}$, and
\begin{equation*}
    (c) = \Bin^{-1}\lr{\frac{n}{2}, \frac{n}{2}\hat{L}(h_{S_1}, S_2), \frac{\delta}{4}} + \Bin^{-1}\lr{\frac{n}{2}, \frac{n}{2}\hat{L}(h_{S_2}, S_1), \frac{\delta}{4}}.
\end{equation*}
\end{theorem}

The proof is postponed to Appendix~\ref{sec:pf_thm:together_refined}.

\section{Experiments}\label{sec:experiments}
We evaluate the performance of the PAC-Bayes-split-$\kl$ inequality in linear classification and in weighted majority vote using several data sets from UCI and LibSVM repositories~\citep{UCI,libsvm}. An overview of the data sets is provided in Appendix~\ref{app:sec:datasets}. For linear classification we reproduce the experimental setup of \citet{MGG20}, and for the weighted majority vote we reproduce the experimental setup of \citet{WMLIS21}. 

\subsection{The Experimental Setup of \citet{MGG20}: Linear Classifiers}\label{sec:LC}

In the first experiment we follow the experimental setup of \citet{MGG20}, who consider binary classification problems with linear classifiers in $\R^d$ and Gaussian priors and posteriors. A classifier $h_w$ associated with a vector $w\in\R^d$ makes a prediction on an input $X$ by $h_w(X)=\1[w^\top X>0]$. The posteriors have the form of Gaussian distributions centered at $w_S\in\R^d$, with covariance $\Sigma_S$ that depends on a sample $S$, $\rho=\mathcal{N}(w_S,\Sigma_S)$. The informed priors $\pi_{S_1}=\mathcal{N}(w_{S_1},\Sigma_{S_1})$ and $\pi_{S_2}=\mathcal{N}(w_{S_2},\Sigma_{S_2})$ are also taken to be Gaussian distributions centered at $w_{S_1}$ and $w_{S_2}$, with covariance $\Sigma_{S_1}$ and $\Sigma_{S_2}$, respectively.  We take the classifier associated with $w_{S_1}$ as the reference classifier $h_{S_1}$ and the classifier associated with $w_{S_2}$ as the reference classifier $h_{S_2}$. More details on the construction are provided in Appendix~\ref{app:sec:LC}.

Figure~\ref{fig:Avg-2Sqn-Ex-IP} compares the PAC-Bayes-Unexpected-Bernstein bound $\PBUB$ and the PAC-Bayes-split-$\kl$ bound $\PBSkl$ with excess losses and informed priors. The ternary random variables in this setup take values in $\{-1,0,1\}$, and we select $\mu$ to be the middle value $0$. Since the PAC-Bayes-$\kl$ bound ($\PBkl$) is one of the tightest known generalization bounds, we take $\PBkl$ with informed priors as a baseline. The details on bound calculation and optimization are provided in Appendix~\ref{app:sec:LC}. In this experiment all the three bounds, $\PBkl$, $\PBUB$, and $\PBSkl$ performed comparably. We believe that the reason is that with informed priors the $\KL(\rho\|\pi)$ term is small. From the relaxation of the $\PBkl$ bound in equation \eqref{eq:pbkl-relaxed}, we observe that a small $\KL(\rho\|\pi)$ term implies smaller difference between fast and slow convergence rates, and thus smaller advantage to bounding the excess loss instead of the raw loss. In other words, we believe that the effect of using informed priors dominates the effect of using excess losses. We note that in order to use excess losses we need to train the reference hypothesis $h^*$ on part of the data and, therefore, training an informed prior on the same data comes at no extra cost.
\begin{figure}[t]
    \centering
    \includegraphics[width=0.82\textwidth]{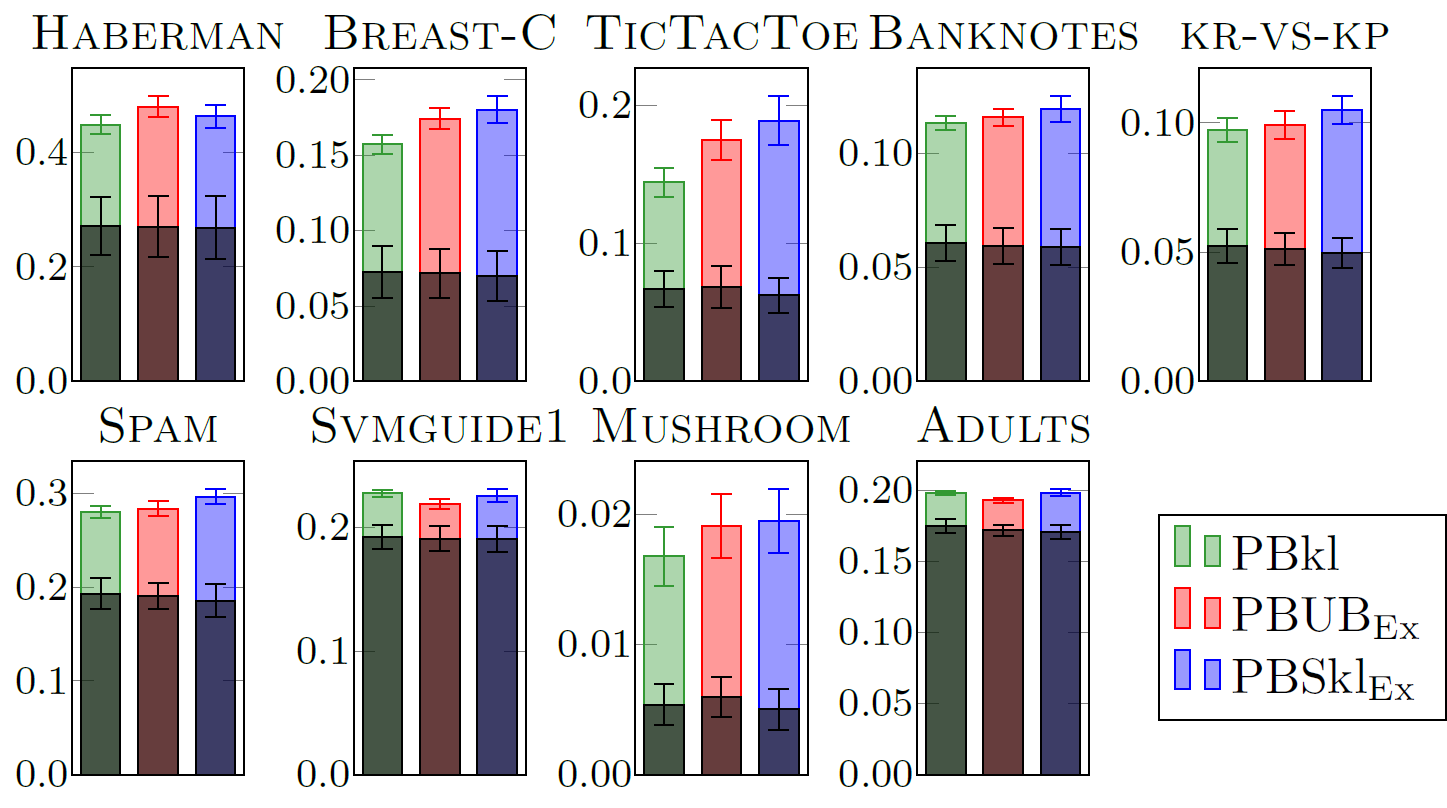}
    \caption{Comparison of the bounds and the test losses of the optimized Gaussian posterior $\rho^*$ generated by $\PBkl$ with informed priors, $\PBUB$ with excess losses and informed priors, and $\PBSkl$ with excess losses and informed priors. The test losses of the corresponding bounds are shown in black. We report the mean and the standard deviation over 20 runs of the experiments.}
    \label{fig:Avg-2Sqn-Ex-IP}
\end{figure}

\subsection{The Experimental Setup of \citet{WMLIS21}: Weighted Majority Vote}\label{sec:Majority_Vote}
In the second experiment we reproduce the experimental setup of \citet{WMLIS21}, who consider multiclass classification by a weighted majority vote. Given an input $X\in\cX$, a hypothesis space $\cH$, and a distribution $\rho$ on $\cH$, a $\rho$-weighted majority vote classifier predicts $\MV_\rho(X)=\argmax_{y\in\cY}\E_\rho[\1[h(X)=y]]$. One of the tightest bound for the majority vote is the tandem bound ($\TND$) proposed by~\citet{MLIS20}, which is based on tandem losses for pairs of hypotheses, $\ell(h(X), h'(X), Y) = \1[h(X)\neq Y]\1[h'(X)\neq Y]$, and the second order Markov's inequality. \citet{WMLIS21} proposed two improved forms of the bound, both based on a parametric form of the Chebyshev-Cantelli inequality. The first, $\CCTND$, using Chebyshev-Cantelli with the tandem losses and the PAC-Bayes-$\kl$ bound for bounding the tandem losses. The second, $\CCPBB$, using tandem losses with an offset, defined by $\ell_\alpha(h(X),h'(X),Y)=(\1[h(X)\neq Y]-\alpha)(\1[h'(X)\neq Y]-\alpha)$ for $\alpha < 0.5$, and PAC-Bayes-Empirical-Bennett inequality for bounding the tandem losses with an offset. We note that while the tandem losses are binary random variables, tandem losses with an offset are ternary random variables taking values in $\{\alpha^2,-\alpha(1-\alpha),(1-\alpha)^2\}$ and, therefore, application of Empirical Bernstein type inequalities makes sense. However, in the experiments of \citeauthor{WMLIS21} $\CCPBB$ lagged behind $\TND$ and $\CCTND$. We replaced PAC-Bayes-Empirical-Bennett with PAC-Bayes-Unexpected-Bernstein ($\CCPBUB$) and PAC-Bayes-split-$\kl$ ($\CCPBSkl$) and showed that the weakness of $\CCPBB$ was caused by looseness of PAC-Bayes-Empirical-Bernstein, and that $\CCPBUB$ and $\CCPBSkl$ lead to tighter bounds that are competitive and sometimes outperforming $\TND$ and $\CCTND$. For the PAC-Bayes-split-$\kl$ bound we took $\mu$ to be the middle value of the tandem loss with an offset, namely, for $\alpha \geq 0$ we took $\mu = \alpha^2$, and for $\alpha < 0$ we took $\mu = -\alpha(1-\alpha)$.

In Figure~\ref{fig:MCE_selected} we present a comparison of the $\TND$, $\CCTND$, $\CCPBB$, $\CCPBUB$, and $\CCPBSkl$ bounds on weighted majority vote of heterogeneous classifiers (Linear Discriminant Analysis, $k$-Nearest Neighbors, Decision Tree,
Logistic Regression, and Gaussian Naive Bayes), which adds the two new bounds, $\CCPBUB$ and $\CCPBSkl$ to the experiment done by \citet{WMLIS21}. A more detailed description of the experiment and results for additional data sets are provided in Appendix~\ref{app:sec:MCE}. We note that $\CCPBUB$ and $\CCPBSkl$ consistently outperform $\CCPBB$, demonstrating that they are more appropriate for tandem losses with an offset. The former two bounds perform comparably to $\TND$ and $\CCTND$, which operate on tandem losses without an offset. In Appendix~\ref{app:sec:RFC} we replicate another experiment of \citeauthor{WMLIS21}, where we use the bounds to reweigh trees in a random forest classifier. The results are similar to the results for heterogeneous classifiers.

\begin{figure}[t]
    \centering
    \includegraphics[width=0.82\textwidth]{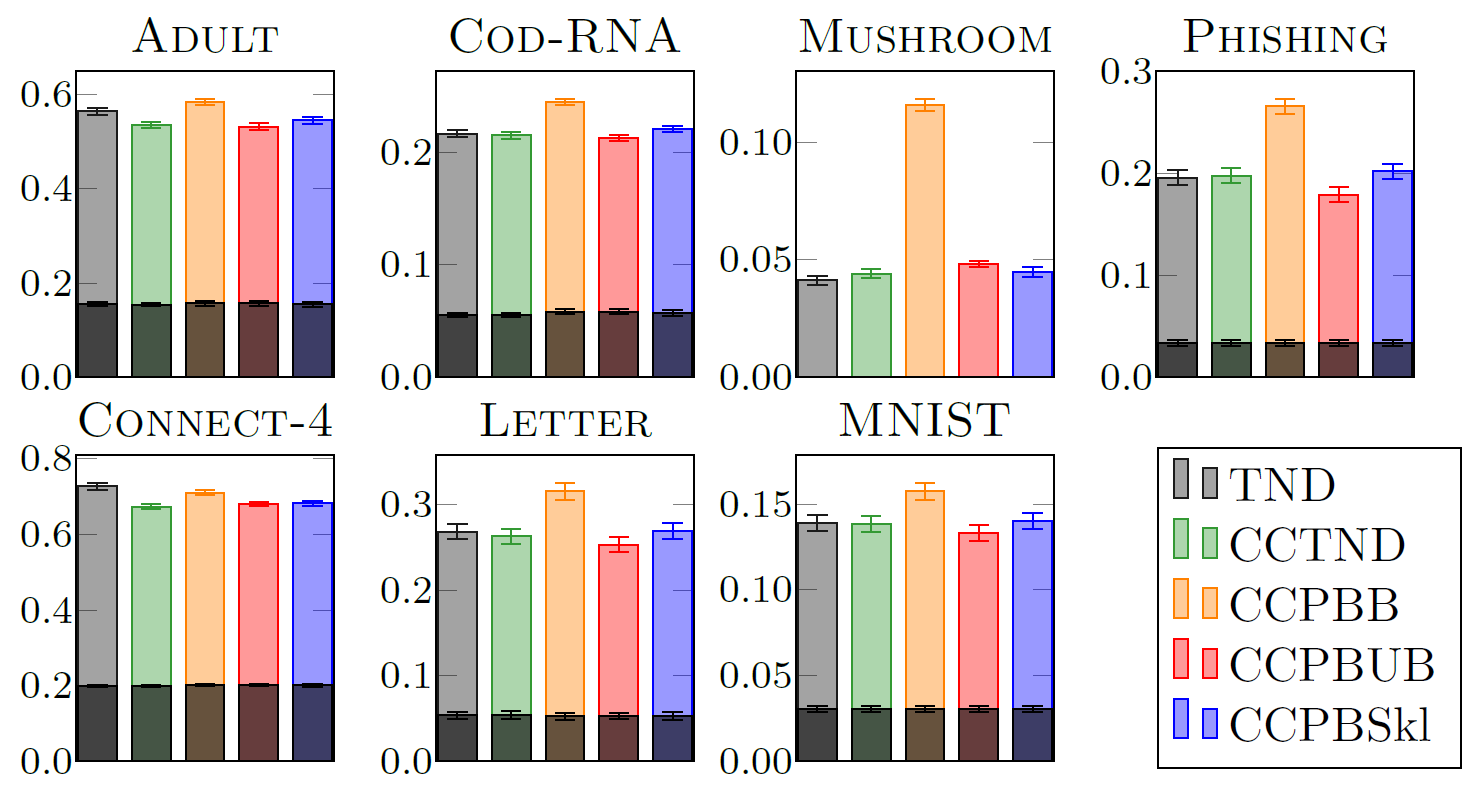}
    \caption{Comparison of the bounds and the test losses of the weighted majority vote on ensembles of heterogeneous classifiers with optimized posterior $\rho^*$ generated by $\TND$, $\CCTND$, $\CCPBB$, $\CCPBUB$, and $\CCPBSkl$. The test losses of the corresponding bounds are shown in black. We report the mean and the standard deviation over 10 runs of the experiments.}
    \label{fig:MCE_selected}
\end{figure}

\section{Discussion}

We have presented the split-$\kl$ and PAC-Bayes-split-$\kl$ inequalities. The inequalities answer a long-standing open question on how to exploit the structure of ternary random variables in order to provide tight concentration bounds. The proposed split-$\kl$ and PAC-Bayes-split-$\kl$ inequalities are as tight for ternary random variables, as the $\kl$ and PAC-Bayes-$\kl$ inequalities are tight for binary random variables.

In our empirical evaluation the split-$\kl$ inequality was always competitive with the $\kl$ and Unexpected Bernstein inequalities and outperformed both in certain regimes, whereas Empirical Bernstein typically lagged behind.  In our experiments in the PAC-Bayesian setting the PAC-Bayes-split-$\kl$ was always comparable to PAC-Bayes-Unexpected-Bernstein, whereas PAC-Bayes-Empirical-Bennett most often lagged behind. The first two inequalities were usually comparable to PAC-Bayes-$\kl$, although in some cases the attempt to exploit low variance did not pay off and PAC-Bayes-$\kl$ outperformed, which is also the trend observed earlier by \citet{MGG20}. To the best of our knowledge, this is the first time when the various approaches to exploitation of low variance were directly compared, and the proposed split-$\kl$ emerged as a clear winner in the basic setting, whereas in the PAC-Bayes setting in our experiments the PAC-Bayes-Unexpected-Bernstein and PAC-Bayes-split-$\kl$ were comparable, and preferable over PAC-Bayes-Empirical-Bernstein and PAC-Bayes-Empirical-Bennett.

\begin{ack}
This project has received funding from European Union’s Horizon 2020 research and innovation
programme under the Marie Skłodowska-Curie grant agreement No 801199. The authors also acknowledge partial support by the Independent Research Fund Denmark, grant number 0135-00259B.
\end{ack}

\newpage 
\small

\bibliographystyle{plainnat}
\bibliography{bibliography.bib}

\begin{thebibliography}{31}
\providecommand{\natexlab}[1]{#1}
\providecommand{\url}[1]{\texttt{#1}}
\expandafter\ifx\csname urlstyle\endcsname\relax
  \providecommand{\doi}[1]{doi: #1}\else
  \providecommand{\doi}{doi: \begingroup \urlstyle{rm}\Url}\fi

\bibitem[Ambroladze et~al.(2007)Ambroladze, Parrado-Hern{\' a}ndez, and
  Shawe-Taylor]{APS07}
Amiran Ambroladze, Emilio Parrado-Hern{\' a}ndez, and John Shawe-Taylor.
\newblock Tighter {PAC-B}ayes bounds.
\newblock In \emph{Advances in Neural Information Processing Systems
  (NeurIPS)}, 2007.

\bibitem[Audibert et~al.(2009)Audibert, Munos, and Szepesv{\' a}ri]{AMS09}
Jean~Yves Audibert, R{\' e}mi Munos, and Csaba Szepesv{\' a}ri.
\newblock Exploration-exploitation trade-off using variance estimates in
  multi-armed bandits.
\newblock \emph{Theoretical Computer Science}, 2009.

\bibitem[Bernstein(1946)]{Ber46}
Sergei~N. Bernstein.
\newblock \emph{Probability Theory}.
\newblock Moscow-Leningrad, 4$^{th}$ edition, 1946.
\newblock In Russian.

\bibitem[Boucheron et~al.(2013)Boucheron, Lugosi, and Massart]{BLM13}
St{\'e}phane Boucheron, G{\'a}bor Lugosi, and Pascal Massart.
\newblock \emph{Concentration Inequalities: {A} Nonasymptotic Theory of
  Independence}.
\newblock Oxford University Press, 2013.

\bibitem[Cesa-Bianchi et~al.(2007)Cesa-Bianchi, Mansour, and Stoltz]{CBMS07}
Nicol{\`o} Cesa-Bianchi, Yishay Mansour, and Gilles Stoltz.
\newblock Improved second-order bounds for prediction with expert advice.
\newblock \emph{Machine Learning}, 66, 2007.

\bibitem[Chang and Lin(2011)]{libsvm}
Chih-Chung Chang and Chih-Jen Lin.
\newblock {LIBSVM}: A library for support vector machines.
\newblock \emph{ACM Transactions on Intelligent Systems and Technology}, 2,
  2011.

\bibitem[Cortes et~al.(2018)Cortes, DeSalvo, Gentile, Mohri, and Yang]{CDG+18}
Corinna Cortes, Giulia DeSalvo, Claudio Gentile, Mehryar Mohri, and Scott Yang.
\newblock Online learning with abstention.
\newblock In \emph{Proceedings of the International Conference on Machine
  Learning (ICML)}, 2018.

\bibitem[Cover and Thomas(2006)]{CT06}
Thomas~M. Cover and Joy~A. Thomas.
\newblock \emph{Elements of Information Theory}.
\newblock Wiley Series in Telecommunications and Signal Processing, 2nd
  edition, 2006.

\bibitem[Dua and Graff(2019)]{UCI}
Dheeru Dua and Casey Graff.
\newblock {UCI} machine learning repository, 2019.
\newblock URL \url{http://archive.ics.uci.edu/ml}.

\bibitem[Dziugaite and Roy(2017)]{DR17}
Gintare~Karolina Dziugaite and Daniel~M. Roy.
\newblock Computing nonvacuous generalization bounds for deep (stochastic)
  neural networks with many more parameters than training data.
\newblock In \emph{UAI}, 2017.

\bibitem[Fan et~al.(2015)Fan, Grama, and Liu]{FGL15}
Xiequan Fan, Ion Grama, and Quansheng Liu.
\newblock Exponential inequalities for martingales with applications.
\newblock \emph{Electronic Journal of Probability}, 20:\penalty0 1--22, 2015.

\bibitem[Florescu and Igel(2018)]{FI18}
Ciprian Florescu and Christian Igel.
\newblock Resilient backpropagation (rprop) for batch-learning in tensorflow.
\newblock In \emph{International Conference on Learning Representations (ICLR)
  Workshop}, 2018.

\bibitem[Foong et~al.(2021)Foong, Bruinsma, Burt, and Turner]{FBBT21}
Andrew Foong, Wessel Bruinsma, David Burt, and Richard Turner.
\newblock How tight can pac-bayes be in the small data regime?
\newblock In \emph{Advances in Neural Information Processing Systems
  (NeurIPS)}, 2021.

\bibitem[Foong et~al.(2022)Foong, Bruinsma, and Burt]{FBB22}
Andrew Y.~K. Foong, Wessel~P. Bruinsma, and David~R. Burt.
\newblock A note on the chernoff bound for random variables in the unit
  interval.
\newblock \emph{arXiv preprint arXiv.2205.07880}, 2022.

\bibitem[Hoeffding(1963)]{Hoe63}
W.~Hoeffding.
\newblock Probability inequalities for sums of bounded random variables.
\newblock \emph{Journal of the American Statistical Association}, 58\penalty0
  (301):\penalty0 13--30, 1963.

\bibitem[Igel and H{\"u}sken(2003)]{IH03}
Christian Igel and Michael H{\"u}sken.
\newblock Empirical evaluation of the improved rprop learning algorithms.
\newblock \emph{Neurocomputing}, 50, 2003.

\bibitem[Langford(2005)]{Lan05}
John Langford.
\newblock Tutorial on practical prediction theory for classification.
\newblock \emph{Journal of Machine Learning Research}, 6, 2005.

\bibitem[Masegosa et~al.(2020)Masegosa, Lorenzen, Igel, and Seldin]{MLIS20}
Andr{\'e}s~R. Masegosa, Stephan~S. Lorenzen, Christian Igel, and Yevgeny
  Seldin.
\newblock Second order {PAC-Bayesian} bounds for the weighted majority vote.
\newblock In \emph{Advances in Neural Information Processing Systems
  (NeurIPS)}, 2020.

\bibitem[Maurer(2004)]{Mau04}
Andreas Maurer.
\newblock A note on the {PAC}-{B}ayesian theorem.
\newblock arXiv preprint cs/0411099, 2004.

\bibitem[Maurer and Pontil(2009)]{MP09}
Andreas Maurer and Massimiliano Pontil.
\newblock Empirical {Bernstein} bounds and sample variance penalization.
\newblock In \emph{Proceedings of the Conference on Learning Theory (COLT)},
  2009.

\bibitem[McAllester(2003)]{McA03}
David McAllester.
\newblock {PAC-Bayesian} stochastic model selection.
\newblock \emph{Machine Learning}, 51, 2003.

\bibitem[Mhammedi et~al.(2019)Mhammedi, Gr{\"u}nwald, and Guedj]{MGG20}
Zakaria Mhammedi, Peter Gr{\"u}nwald, and Benjamin Guedj.
\newblock {PAC-Bayes} un-expected {Bernstein} inequality.
\newblock In \emph{Advances in Neural Information Processing Systems
  (NeurIPS)}, 2019.

\bibitem[Mnih et~al.(2008)Mnih, Szepesv{\' a}ri, and Audibert]{MSA08}
Volodymyr Mnih, Csaba Szepesv{\' a}ri, and Jean-Yves Audibert.
\newblock Empirical {Bernstein} stopping.
\newblock In \emph{Proceedings of the International Conference on Machine
  Learning (ICML)}, 2008.

\bibitem[Seeger(2002)]{See02}
Matthias Seeger.
\newblock {PAC-Bayesian} generalization error bounds for {Gaussian} process
  classification.
\newblock \emph{Journal of Machine Learning Research}, 3, 2002.

\bibitem[Seldin et~al.(2012)Seldin, Laviolette, Cesa-Bianchi, Shawe-Taylor, and
  Auer]{SLCB+12}
Yevgeny Seldin, Fran\c{c}ois Laviolette, Nicol{\`o} Cesa-Bianchi, John
  Shawe-Taylor, and Peter Auer.
\newblock {PAC-Bayesian} inequalities for martingales.
\newblock \emph{IEEE Transactions on Information Theory}, 58, 2012.

\bibitem[Thiemann et~al.(2017)Thiemann, Igel, Wintenberger, and Seldin]{TIWS17}
Niklas Thiemann, Christian Igel, Olivier Wintenberger, and Yevgeny Seldin.
\newblock A strongly quasiconvex {PAC-Bayesian} bound.
\newblock In \emph{Proceedings of the International Conference on Algorithmic
  Learning Theory (ALT)}, 2017.

\bibitem[Thulasidasan et~al.(2019)Thulasidasan, Bhattacharya, Bilmes,
  Chennupati, and Mohd{-}Yusof]{TBB+19}
Sunil Thulasidasan, Tanmoy Bhattacharya, Jeff~A. Bilmes, Gopinath Chennupati,
  and Jamal Mohd{-}Yusof.
\newblock Combating label noise in deep learning using abstention.
\newblock In \emph{Proceedings of the International Conference on Machine
  Learning (ICML)}, 2019.

\bibitem[Tolstikhin and Seldin(2013)]{TS13}
Ilya Tolstikhin and Yevgeny Seldin.
\newblock {PAC-Bayes-Empirical-Bernstein} inequality.
\newblock In \emph{Advances in Neural Information Processing Systems
  (NeurIPS)}, 2013.

\bibitem[Wintenberger(2017)]{Win17}
Olivier Wintenberger.
\newblock Optimal learning with {Bernstein} online aggregation.
\newblock \emph{Machine Learning}, 106, 2017.

\bibitem[Wu et~al.(2021)Wu, Masegosa, Lorenzen, Igel, and Seldin]{WMLIS21}
Yi-Shan Wu, Andres Masegosa, Stephan Lorenzen, Christian Igel, and Yevgeny
  Seldin.
\newblock Chebyshev-cantelli pac-bayes-bennett inequality for the weighted
  majority vote.
\newblock In \emph{Advances in Neural Information Processing Systems
  (NeurIPS)}, 2021.

\bibitem[Zhou et~al.(2019)Zhou, Veitch, Austern, Adams, and Orbanz]{ZVAAO19}
Wenda Zhou, Victor Veitch, Morgane Austern, Ryan~P. Adams, and Peter Orbanz.
\newblock Non-vacuous generalization bounds at the imagenet scale: a
  {PAC}-bayesian compression approach.
\newblock In \emph{ICLR}, 2019.

\end{thebibliography}

\normalsize

\newpage
\appendix

\section{Unexpected Bernstein Inequality}
\subsection{A Proof of the Unexpected Bernstein Inequality (Theorem~\ref{thm:Unexpected-Bernstein})}\label{sec:pf_thm:Unexpected-Bernstein}
The proof is based on the Unexpected Bernstein lemma.
\begin{lemma}[Unexpected Bernstein lemma~\citep{FGL15,MGG20}]\label{lem:unepected-bernstein}
Let $Z_1,\cdots,Z_n$ be i.i.d. random variables bounded from above by $b>0$, and assume that $\sum_{i=1}^n Z_i^2$ is finite. Let $\psi(u):=u-\ln(1+u)$ for $u\in\R$. Then for any $\gamma\in(0,\frac{1}{b})$:
\[
\E[e^{\gamma\sum_{i=1}^n(\E[Z_i]-Z_i)-\frac{\psi(-b\gamma)}{b^2}\sum_{i=1}^nZ_i^2}]\leq 1.
\]
\end{lemma}

\begin{proof}[Proof of Theorem~\ref{thm:Unexpected-Bernstein}]
Recall that by the assumption of the theorem $Z_1,\dots,Z_n$ are i.i.d., bounded from above by $b>0$, and that $p=\E[Z_i]$ for all $i$, $\hat{p}=\frac{1}{n}\sum_{i=1}^n Z_i$, and $\hat{\sigma}=\frac{1}{n}\sum_{i=1}^n Z_i^2$. For any $\gamma\in(0,1/b)$ we have:
\begin{align*}
    \P[p- \hat{p}-\frac{\psi(-\gamma b)}{\gamma b^2}\hat{\sigma}\geq \varepsilon] &= \P[\sum_{i=1}^n \lr{ \E[Z_i] - Z_i -\frac{\psi(-\gamma b)}{\gamma b^2}Z_i^2 }\geq n\varepsilon]\\
    &= \P[e^{\gamma\sum_{i=1}^n \lr{ \E[Z_i] - Z_i -\frac{\psi(-\gamma b)}{\gamma b^2}Z_i^2 }}\geq e^{\gamma n\varepsilon} ]\\
    &\leq \E[e^{\gamma\sum_{i=1}^n \lr{ \E[Z] - Z_i -\frac{\psi(-\gamma b)}{\gamma b^2}Z_i^2 }}]/e^{\gamma n\varepsilon}\\
    &\leq e^{-\gamma n\varepsilon},
\end{align*}
where the first inequality is by application of Markov's inequality and the second inequality is by application of Lemma~\ref{lem:unepected-bernstein}. By taking $\delta = e^{-\gamma n \varepsilon}$ and solving for $\varepsilon$ we complete the proof.
\end{proof}

\subsection{A relaxation of the Unexpected Bernstein lemma}\label{sec:comp:Unexpected-Bernstein}

We show that a concentration inequality introduced by \citet{CBMS07} yields a relaxation of the Unexpected Bernstein Lemma. The inequality of \citeauthor{CBMS07} can be used to directly derive a relaxed version of the Unexpected Bernstein lemma, but as we show the result is weaker than the Unexpected Bernstein lemma.  \citet[Lemma 1]{CBMS07} have shown that 
\begin{equation}\label{eq:CBMS07orig}
    \forall \gamma \geq -1/2: ~~ \gamma-\gamma^2\leq \ln(1+\gamma).
\end{equation}
Thus,
\begin{equation}\label{eq:CBMS07}
    \forall \gamma\leq 1/2: ~~ -\gamma^2\leq \gamma + \ln(1-\gamma) = -\psi(-\gamma).
\end{equation}
This gives a relaxed version of the Unexpected Bernstein lemma. For simplicity, we present it with $b=1$.
\begin{lemma}[Relaxed Unexpected Bernstein lemma]\label{lem:relaxed-unepected-bernstein}
Let $Z_1,\cdots,Z_n$ be i.i.d.\ random variables bounded from above by $1$, and assume that $\sum_{i=1}^n Z_i^2$ is finite. Then for any $\gamma\in[0,\frac{1}{2}]$:
\[
\E[e^{\gamma\sum_{i=1}^n(\E[Z_i]-Z_i)-\gamma^2\sum_{i=1}^nZ_i^2}]\leq 1.
\]
\end{lemma}

\begin{proof}
By \eqref{eq:CBMS07} and Lemma~\ref{lem:unepected-bernstein} we have
\[
\E[e^{\gamma\sum_{i=1}^n(\E[Z_i]-Z_i)-\gamma^2\sum_{i=1}^nZ_i^2}] \leq \E[e^{\gamma\sum_{i=1}^n(\E[Z_i]-Z_i)-\psi(-\gamma)\sum_{i=1}^nZ_i^2}] \leq 1.
\]
\end{proof}
We note that it is possible to prove Lemma~\ref{lem:relaxed-unepected-bernstein} directly by using inequality \eqref{eq:CBMS07orig} and without using Lemma~\ref{lem:unepected-bernstein}, as done by \citet{Win17}. The first inequality in our proof of Lemma~\ref{lem:relaxed-unepected-bernstein} shows that it is a relaxation of Lemma~\ref{lem:unepected-bernstein}.

\section{A Proof of the PAC-Bayes Unexpected Bernstein Inequality (Theorem~\ref{thm:PB-unexpected-bernstein})}\label{sec:pf_thm:PB-unexpected-bernstein}

The proof is based on using the Unexpected Bernstein lemma within a standard change of measure argument cited in Lemma~\ref{lem:PAC-Bayes} below. We cite the version before the expectations of $S'$ and $\pi$ are exchanged, which is an intermediate step in the proof of~\citet[Lemma 1]{TS13}.
\begin{lemma}[PAC-Bayes Lemma~\citep{TS13} ]
\label{lem:PAC-Bayes}
For any function $f_n:\cH\times ({\cal X}\times{\cal Y})^n\to \R$ and for any distribution $\pi$ on $\cH$
, with probability at least $1-\delta$ over a random draw of $S$, for all distributions $\rho$ on $\cH$ simultaneously:
\[
\E_\rho[f_n(h,S)] \leq \KL(\rho\|\pi) + \ln \frac{1}{\delta} + \ln\E_{S'}[\E_\pi[e^{f_n(h,S')}]].
\]
\end{lemma}

\begin{proof}[Proof of Theorem~\ref{thm:PB-unexpected-bernstein}]
Let $f_n(h,S)=\gamma n\lr{\tilde{L}(h)-\hat{\tilde{L}}(h,S)}-\frac{\psi(-b\gamma)}{b^2}n\hat{\tilde{\Var}}(h,S)$. Since $\pi$ is independent of $S$ by assumption, we can exchange the expectations of $S$ and $\pi$.  Then by Lemma~\ref{lem:unepected-bernstein} we have $\E[e^{f_n(h,S)}]\leq 1$. By plugging this into Lemma~\ref{lem:PAC-Bayes} and dividing both sides by $\gamma n$, we complete the proof.
\end{proof}




\section{Proof of Theorem~\ref{thm:together_refined}}\label{sec:pf_thm:together_refined}

To prove the theorem, we need the test set bound (Theorem~\ref{thm:test-set-bound}), the PAC-Bayes Lemma (Lemma~\ref{lem:PAC-Bayes}), and the following lemma.
\begin{lemma}[\citep{Mau04}]\label{lem:Note-PB-Maurer}
    Let $X_1,\cdots,X_n$ be i.i.d.\ random variables with mean $p$ and bounded in the $[0,1]$ interval. Let $\hat{p}=\frac{1}{n}\sum_{i=1}^nX_i$ be the empirical mean. Then:
    \[
    \E[e^{n\kl(\hat{p}\|p)}]\leq 2\sqrt{n}.
    \]
\end{lemma}

\begin{proof}[Proof of Theorem~\ref{thm:together_refined}]
Recall that 
\begin{equation}\label{eqpf:refined_EL&IP}
    \E_\rho[L(h)]  = \frac{1}{2}\E_\rho[\Delta_L(h,h_{S_1})] + \frac{1}{2}\E_\rho[\Delta_L(h,h_{S_2})] + \frac{1}{2}\lr{L(h_{S_1})+L(h_{S_2})}.
\end{equation}
First, by applying Theorem~\ref{thm:test-set-bound} to $L(h_{S_1})$ and $L(h_{S_2})$, respectively, we have:
\begin{equation}\label{eqpf:refined_Lhs1}
        \P[L(h_{S_1}) \geq  \Bin^{-1}\lr{\frac{n}{2}, \frac{n}{2}\hat{L}(h_{S_1}, S_2), \delta}]\leq \delta
\end{equation}
and 
\begin{equation}\label{eqpf:refined_Lhs2}
     \P[L(h_{S_2}) \geq \Bin^{-1}\lr{\frac{n}{2}, \frac{n}{2}\hat{L}(h_{S_2}, S_1), \delta}]\leq \delta.
\end{equation}
Next, since  \[\E_\rho[\Delta_L(h,h_{S_1})]=\mu + \E_\rho[\Delta_L^+(h, h_{S_1})] - \E_\rho[\Delta_L^-(h,h_{S_1})]
\] 
and 
\[\E_\rho[\Delta_L(h,h_{S_2})]=\mu + \E_\rho[\Delta_L^+(h,h_{S_2})] - \E_\rho[\Delta_L^-(h,h_{S_2})],
\]
for any $\mu\in[a,b]$ we have
\begin{multline}
\frac{1}{2}\E_\rho[\Delta_L(h,h_{S_1})] + \frac{1}{2}\E_\rho[\Delta_L(h,h_{S_2})]\\
=\mu + \lr{\frac{1}{2}\E_\rho[\Delta_L^+(h, h_{S_1})] + \frac{1}{2}\E_\rho[\Delta_L^+(h,h_{S_2})]} - \lr{\frac{1}{2}\E_\rho[\Delta_L^-(h, h_{S_1})]  + \frac{1}{2}\E_\rho[\Delta_L^-(h,h_{S_2})]}.\label{eqpf:refined_rewriteby+-}
\end{multline}
Let $\pi=\frac{1}{2}\pi_{S_1}+\frac{1}{2}\pi_{S_2}$, 
and let $S_*$ be either $S_1$ or $S_2$ and $\bar{S}_*=S\backslash S_*$. If $h$ is sampled from $\pi_{S_*}$, we take $h_{S_*}$ as a reference hypothesis and estimate the excess loss on $\bar{S}_*$. Then,

\begin{align*}
    \E_{S}\E_{\pi}\lrs{e^{\frac{n}{2}\kl\lr{\frac{\Delta_{\hat{L}}^+(h,h_{S_*},\bar{S}_*)}{1-\mu}\middle\|\frac{\Delta_L^+(h,h_{S_*})}{1-\mu} }}} &= \frac{1}{2} \sum\limits_{i=1,2} \E_S \E_{\pi_{S_i}} \lrs{e^{\frac{n}{2}\kl\lr{\frac{\Delta_{\hat{L}}^+(h,h_{S_i},\bar{S}_i)}{1-\mu}\middle\|\frac{\Delta_L^+(h,h_{S_i})}{1-\mu} }}}\\
    &=\frac{1}{2} \sum\limits_{i=1,2} \E_{S_i} \E_{\pi_{S_i}} \E_{\bar{S}_i} \lrs{e^{\frac{n}{2}\kl\lr{\frac{\Delta_{\hat{L}}^+(h,h_{S_i},\bar{S}_i)}{1-\mu}\middle\|\frac{\Delta_L^+(h,h_{S_i})}{1-\mu} }}}\\
    &\leq 2\sqrt{n/2},
\end{align*}
where the second equality is due to the fact that $\pi_{S_*}$ is independent of $\bar{S}_*$ so they are exchangeable, and the inequality follows by Lemma~\ref{lem:Note-PB-Maurer}. 

Therefore, by applying Lemma~\ref{lem:PAC-Bayes} with $f(h,S)=\frac{n}{2}\kl\lr{\frac{\Delta_{\hat{L}}^+(h,h_{S_*},\bar{S}_*)}{1-\mu}\middle\|\frac{\Delta_L^+(h,h_{S_*})}{1-\mu} }$, we have with probability at least $1-\delta$ over $S$, for all $\rho$ on $\cH$ simultaneously:
\[
\E_\rho\lrs{\frac{n}{2}\kl\lr{ \frac{\Delta_{\hat{L}}^+(h,h_{S_*},\bar{S}_* )}{1-\mu}\middle\|\frac{\Delta_L^+(h,h_{S_*})}{1-\mu} } } \leq \KL(\rho\|\pi) + \ln\frac{2\sqrt{n/2}}{\delta}.
\]

By the convexity of $\KL$, we further have

\[
\kl\lr{\E_\rho\lrs{\frac{\Delta_{\hat{L}}^+(h,h_{S_*},\bar{S}_* )}{1-\mu} } \middle\| \E_\rho\lrs{\frac{\Delta_L^+(h,h_{S_*})}{1-\mu} } } \leq \E_\rho\lrs{\kl\lr{ \frac{\Delta_{\hat{L}}^+(h,h_{S_*},\bar{S}_* )}{1-\mu}\middle\|\frac{\Delta_L^+(h,h_{S_*})}{1-\mu} } },
\]
which together gives with probability at least $1-\delta$ over $S$, for all $\rho$ on $\cH$ simultaneously:
\[
\kl\lr{\E_\rho\lrs{\frac{\Delta_{\hat{L}}^+(h,h_{S_*},\bar{S}_* )}{1-\mu} } \middle\| \E_\rho\lrs{\frac{\Delta_L^+(h,h_{S_*})}{1-\mu} } } \leq \frac{\KL(\rho\|\pi) + \ln\frac{2\sqrt{n/2}}{\delta}}{n/2}.
\]

Similarly, $\Delta_{\hat{L}}^-(h,h_{S_*},\bar{S}_*)$ and $\Delta_L^-(h,h_{S_*})$ also satisfy with probability at least $1-\delta$ over $S$, for all $\rho$ on $\cH$ simultaneously:
\[
\kl\lr{\E_\rho\lrs{\frac{\Delta_{\hat{L}}^-(h,h_{S_*},\bar{S}_* )}{\mu+1} } \middle\| \E_\rho\lrs{\frac{\Delta_L^-(h,h_{S_*})}{\mu+1} } } \leq \frac{\KL(\rho\|\pi) + \ln\frac{2\sqrt{n/2}}{\delta}}{n/2}.
\]

Let $\rho=\frac{1}{2}\rho_1+\frac{1}{2}\rho_2$ be constructed in a similar way to $\pi$, where $\rho_1$ and $\rho_2$ are probability distributions on $\cH$. If $h$ is sampled from $\rho_*$, then we take $h_{S_*}$ as a reference hypothesis and estimate the excess loss on $\bar{S}_*$. In our case, $\rho_1=\rho_2=\rho$. Let $\Delta^\circ$ denote either $\Delta^+$ or $\Delta^-$. Then,
\[
\E_\rho[\Delta_{L}^\circ(h,h_{S_*})] = \frac{1}{2} \E_{\rho}[\Delta_{L}^\circ(h,h_{S_1})] + \frac{1}{2}\E_{\rho}[\Delta_{L}^\circ(h,h_{S_2})]
\]
and
\[
\E_\rho[\Delta_{\hat{L}}^\circ(h,h_{S_*},\bar{S}_*)] = \frac{1}{2} \E_{\rho}[\Delta_{\hat{L}}^\circ(h,h_{S_1},S_2)] + \frac{1}{2}\E_{\rho}[\Delta_{\hat{L}}^\circ(h,h_{S_2},S_1)].
\]


By taking the inverse of $\kl$, we obtain that with probability at least $1-\delta$ over $S$, for all $\rho$ on $\cH$ simultaneously:
\begin{align}\label{eqpf:refined_EL_plus}
    \lefteqn{\frac{1}{2} \frac{\E_{\rho}[\Delta_{L}^+(h,h_{S_1})]}{1-\mu} + \frac{1}{2}\frac{\E_{\rho}[\Delta_{L}^+(h,h_{S_2})]}{1-\mu}} \nonumber \\
    &\leq \kl^{-1,+}\lr{\frac{1}{2} \frac{\E_{\rho}[\Delta_{\hat{L}}^+(h,h_{S_1},S_2)]}{1-\mu} + \frac{1}{2}\frac{\E_{\rho}[\Delta_{\hat{L}}^+(h,h_{S_2},S_1)]}{1-\mu}, \frac{\KL(\rho\|\pi) + \ln\frac{2\sqrt{n/2}}{\delta}}{n/2} },
\end{align}
and with the same probability
\begin{align}\label{eqpf:refined_EL_minus}
    \lefteqn{\frac{1}{2} \frac{\E_{\rho}[\Delta_{L}^-(h,h_{S_1})]}{\mu+1} + \frac{1}{2}\frac{\E_{\rho}[\Delta_{L}^-(h,h_{S_2})]}{\mu+1}} \nonumber \\
    &\geq \kl^{-1,-}\lr{\frac{1}{2} \frac{\E_{\rho}[\Delta_{\hat{L}}^-(h,h_{S_1},S_2)]}{\mu+1} + \frac{1}{2}\frac{\E_{\rho}[\Delta_{\hat{L}}^-(h,h_{S_2},S_1)]}{\mu+1}, \frac{\KL(\rho\|\pi) + \ln\frac{2\sqrt{n/2}}{\delta}}{n/2} }.
\end{align}
Thus, we can bound Eq. \eqref{eqpf:refined_rewriteby+-} by Eq.\eqref{eqpf:refined_EL_plus} and Eq.\eqref{eqpf:refined_EL_minus}. By replacing Eq.\eqref{eqpf:refined_EL&IP} by the upper bound of each term and taking a union bound, we complete the proof.
\end{proof}

\section{Empirical Comparison}\label{app:sec:empirical-comparison}
We present more results on empirical comparison of the concentration inequalities: the $\kl$, the Empirical Bernstein, the Unexpected Bernstein, and the split-$\kl$. In particular, Section~\ref{app:sec:empirical-comparison-ternary} expands the empirical comparison in Section~\ref{sec:empirical_comparison} in the body for ternary random variables, and Section~\ref{app:sec:empirical-comparison-bounded} studies the empirical comparison of bounded random variables.   The source code for replicating the experiments is available at Github\footnote{\url{https://github.com/YiShanAngWu/Split-KL-R/tree/main/simulation}}. 

\subsection{Ternary Random Variables}\label{app:sec:empirical-comparison-ternary}
In this section, we follow the settings and the parameters in Section~\ref{sec:empirical_comparison}, considering $n$ i.i.d. samples taking values in $\{-1,0,1\}$. For completeness, Figure~\ref{fig:app:sim:n100_rate0.5} and Figure~\ref{fig:app:sim:n100_rate0.99} repeats Figures~\ref{fig:sim:n100_rate0.5} and ~\ref{fig:sim:n100_rate0.99} while we add Figure~\ref{fig:app:sim:n100_rate0.01}, where the probability is defined by $p_1=0.01(1-p_0)$ and $p_{-1}=0.99(1-p_0)$. In this case, the $\kl$ starts well for $p_0$ close to zero, but similar to the case in Figure~\ref{fig:app:sim:n100_rate0.99} falls behind due to its inability of properly handling the values inside the interval. The Unexpected Bernstein and the Empirical Bernstein perform similarly when $p_0$ is small in Figure~\ref{fig:app:sim:n100_rate0.99}  since the bounds are cut to $1$, while Unexpected Bernstein falls behind Empirical Bernstein when $p_0$ is small in Figure~\ref{fig:app:sim:n100_rate0.01} due to the uncentered second moment. The split-$\kl$ matches, and in many cases outperforms, the tightest bounds.
\begin{figure}[t]
	\begin{subfigure}[b]{.32\textwidth}
		\includegraphics[width=\textwidth]{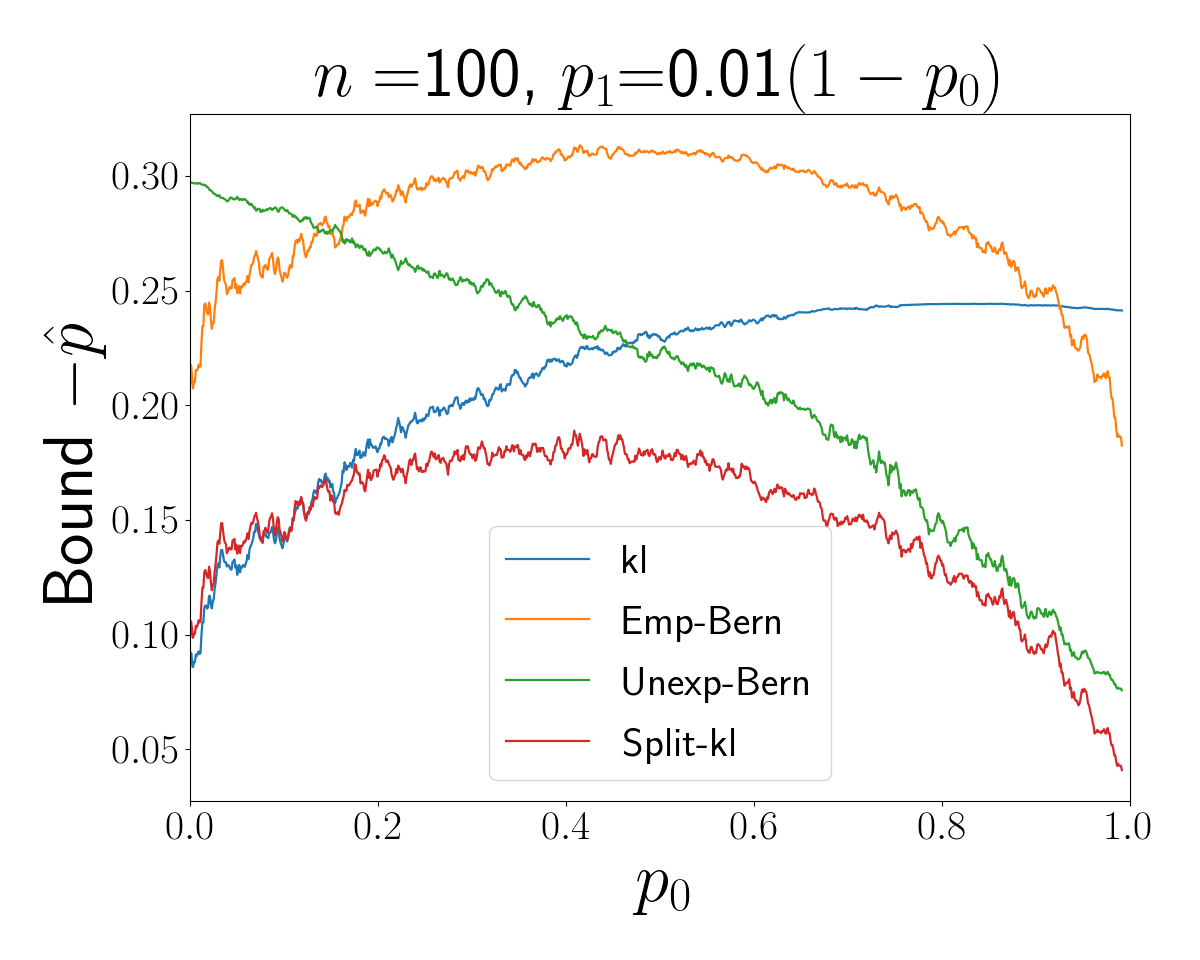}
		\caption{}
		\label{fig:app:sim:n100_rate0.01}
	\end{subfigure}
	\hfill
	\begin{subfigure}[b]{.32\textwidth}
		\includegraphics[width=\textwidth]{experiment/simulation/n100_rate0.5_RV-101.png}
		\caption{}
		\label{fig:app:sim:n100_rate0.5}
	\end{subfigure}
	\hfill
	\begin{subfigure}[b]{.32\textwidth}
		\includegraphics[width=\textwidth]{experiment/simulation/n100_rate0.99_RV-101.png}
		\caption{}
		\label{fig:app:sim:n100_rate0.99}
	\end{subfigure}
	\caption{Comparison of the concentration bounds with $n=100$, $\delta=0.05$, and (a) $p_1=0.01(1-p_0)$ and $p_{-1}=0.99(1-p_0)$, (b)  $p_{-1}=p_1=0.5(1-p_0)$, (c)  $p_1=0.99(1-p_0)$ and $p_{-1}=0.01(1-p_0)$.}
	\label{fig:app:ternary:n100}
\end{figure}

Figure~\ref{fig:app:ternary:n1000} has the same setting with a larger number of samples $n=1000$. The trends of the bounds are similar to Figure~\ref{fig:app:ternary:n100}. However, the Empirical Bernstein performs better than the Unexpected Bernstein in Figure~\ref{fig:app:sim:n1000_rate0.01} and Figure~\ref{fig:app:sim:n1000_rate0.99} when $p_0$ is less than $0.6$. In both cases, split-$\kl$ keeps its leading position. When $p_1=p_{-1}=(1-p_0)/2$ (Figure~\ref{fig:app:sim:n1000_rate0.5}), as $p_0=0$, the random variable becomes Bernoulli and, as expected, the $\kl$ bound performs the best, followed by the split-$\kl$, and then the two Bernstein bounds. As $p_0$ grows larger, the $\kl$ bound falls behind the other three bounds due to inability of properly handling the values inside the interval. The Unexpected Bernstein, the Empirical Bernstein and the split-$\kl$ perform similarly well.

\begin{figure}[t]
	\begin{subfigure}[b]{.32\textwidth}
		\includegraphics[width=\textwidth]{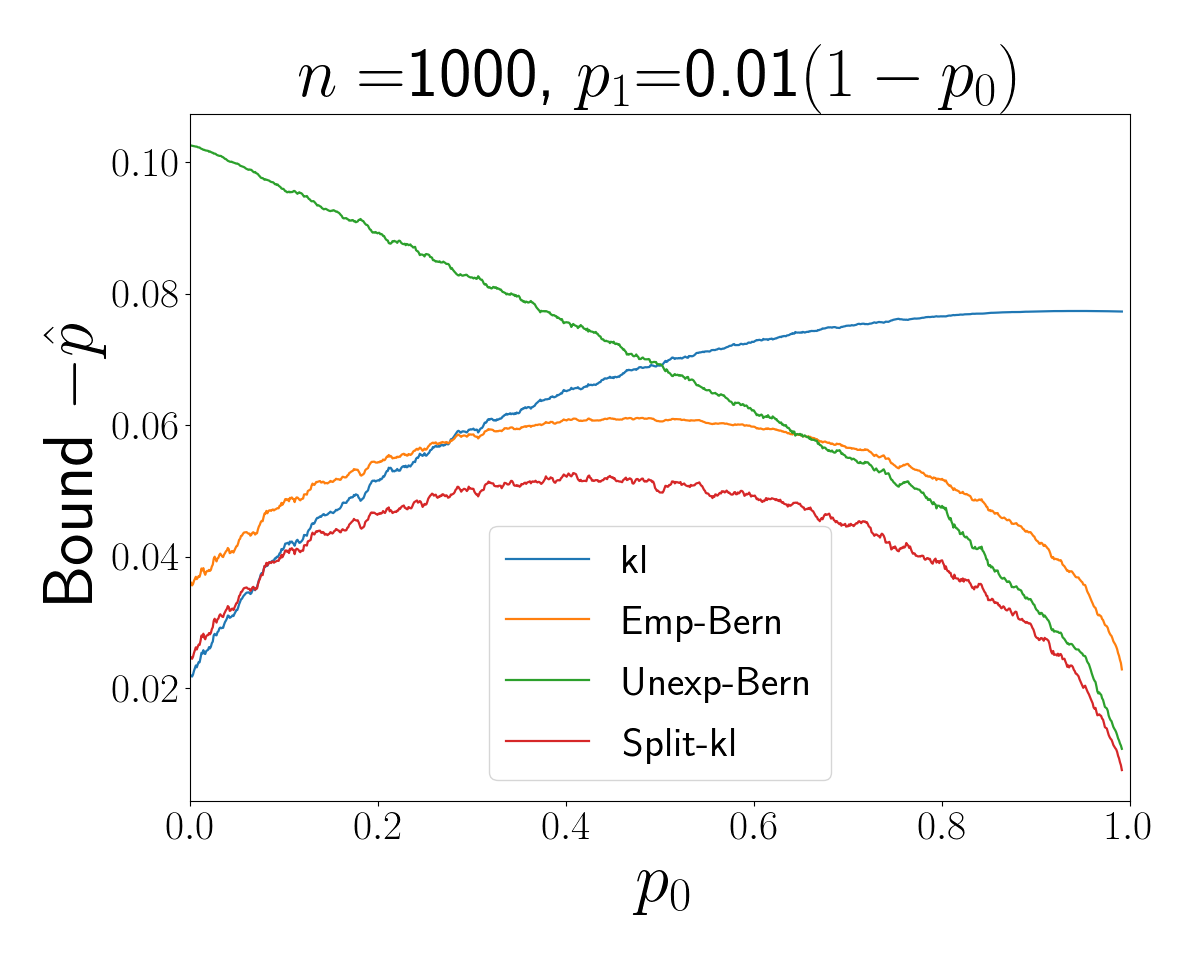}
		\caption{}
		\label{fig:app:sim:n1000_rate0.01}
	\end{subfigure}
	\hfill
	\begin{subfigure}[b]{.32\textwidth}
		\includegraphics[width=\textwidth]{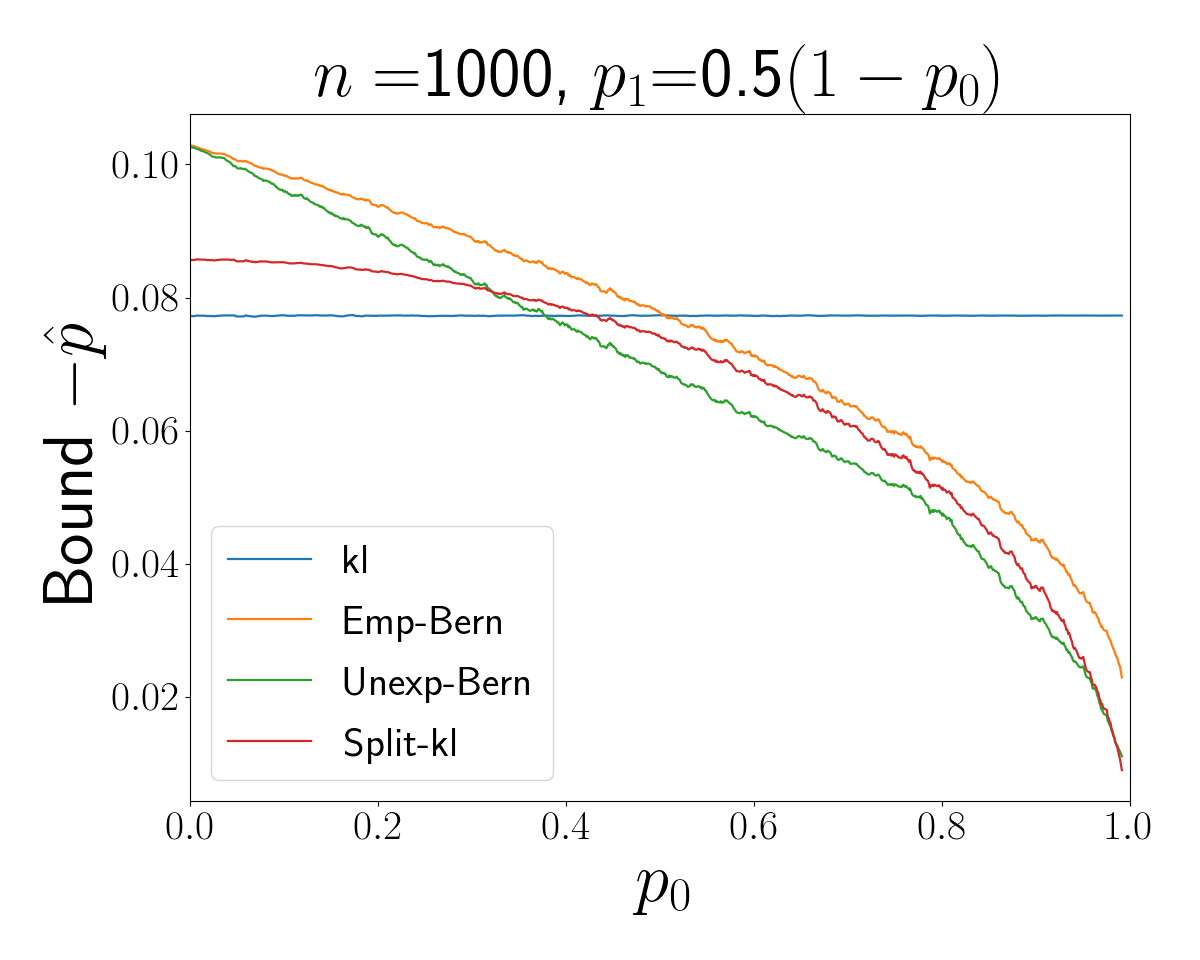}
		\caption{}
		\label{fig:app:sim:n1000_rate0.5}
	\end{subfigure}
	\hfill
	\begin{subfigure}[b]{.32\textwidth}
		\includegraphics[width=\textwidth]{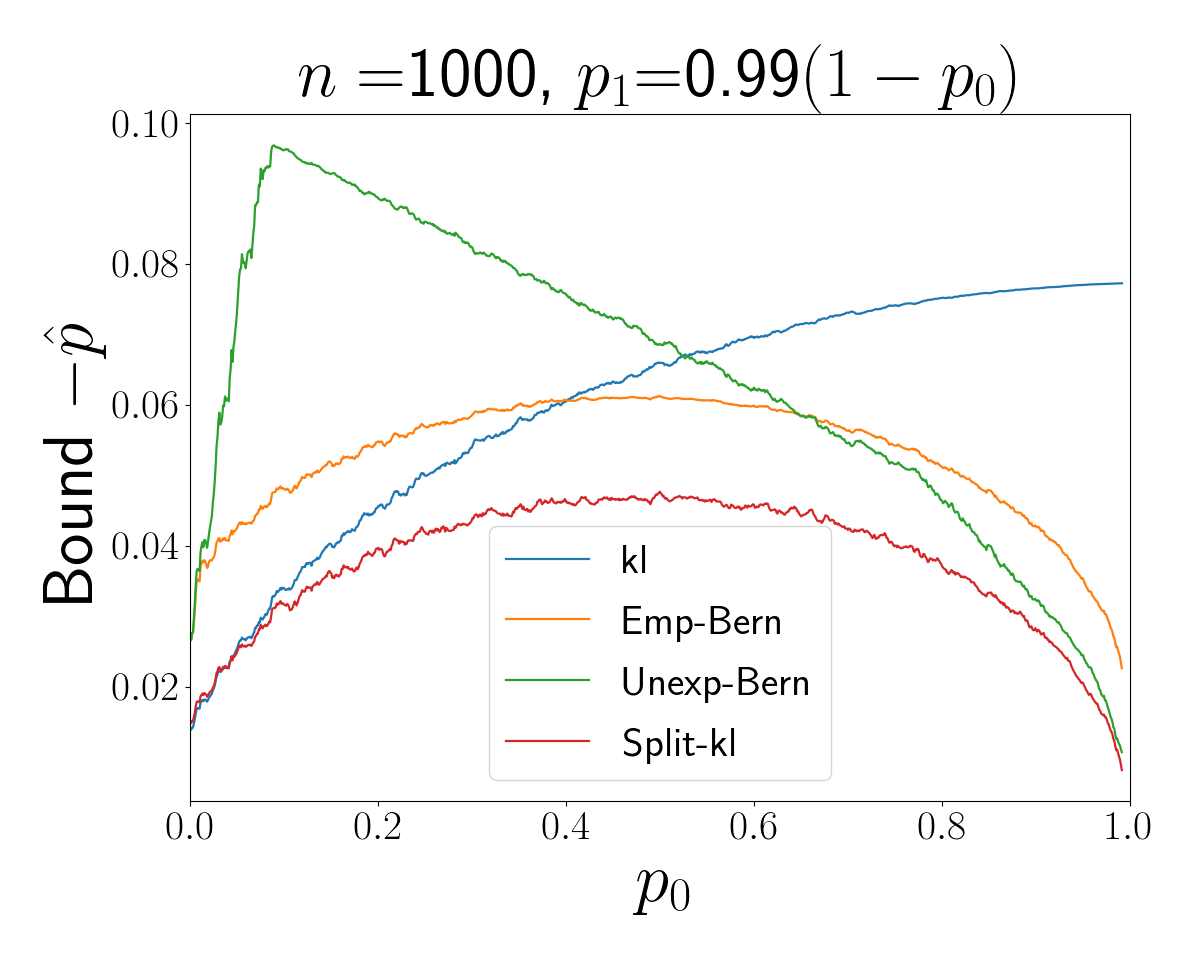}
		\caption{}
		\label{fig:app:sim:n1000_rate0.99}
	\end{subfigure}
	\caption{Comparison of the concentration bounds with $n=1000$, $\delta=0.05$, and (a) $p_1=0.01(1-p_0)$ and $p_{-1}=0.99(1-p_0)$, (b) $p_{-1}=p_1=0.5(1-p_0)$, (c)   $p_1=0.99(1-p_0)$ and $p_{-1}=0.01(1-p_0)$.}
	\label{fig:app:ternary:n1000}
\end{figure}

\subsection{Bounded Random Variables}\label{app:sec:empirical-comparison-bounded}
In this section, we study a more general setting, where the i.i.d.\ random variables $Z_1,\cdots,Z_n$ taking values in $[0,1]$. Naturally, we consider the random variables following beta distribution with parameters $\alpha>0$ and $\beta>0$, where the mean $p=\frac{\alpha}{\alpha+\beta}$ and the variance $\V[Z]=\frac{\alpha\beta}{(\alpha+\beta)^2(\alpha+\beta+1)}$. For the Empirical Bernstein bound, we take $a=0$ and $b=1$. For the Unexpected Bernstein bound we take a grid of $\gamma \in \{1/(2b), \cdots, 1/(2^k b)\}$ for $b=1$, $k=\lceil \log_2(\sqrt{n/\ln(1/\delta)}/2) \rceil$, and a union bound over the grid, as in Section~\ref{sec:empirical_comparison}.  For the split-$\kl$ bound we take $\mu$ to be the middle value $0.5$. Again, in the experiments we take $\delta=0.05$ and cut the bounds to $1$.

In Figure~\ref{fig:app:bounded:ConstM} we take $\alpha=\beta$ in an interval of $[0.01, 10]$. The mean is a constant $p=0.5$ throughout the interval and the variance is in an interval of $[0,012, 0.245]$, where a small $\alpha$ and $\beta$ corresponds to a large variance, and a large $\alpha$ and $\beta$ corresponds to a small variance. We plot the difference between the values of the bounds and $\hat p$ as a function of the variance $\V[Z]$. Since the true mean is a constant, the $\kl$ bound is also almost a constant throughout the interval. When the variance large, the $\kl$ bound performs the best, followed by spli-$\kl$ and the two Bernstein bounds. When the variance is small, the Empirical Bernstein bound exploit the low variance and outperform all the others when the number of samples is sufficiently large. The Unexpected Bernstein falls behind due the uncentered second moment. The split-$\kl$ bound is comparable to the $\kl$ bound when the variance is large and also comparable to the tightest bound when the variance is small.

\begin{figure}[t]
	\begin{subfigure}[b]{.48\textwidth}
		\includegraphics[width=\textwidth]{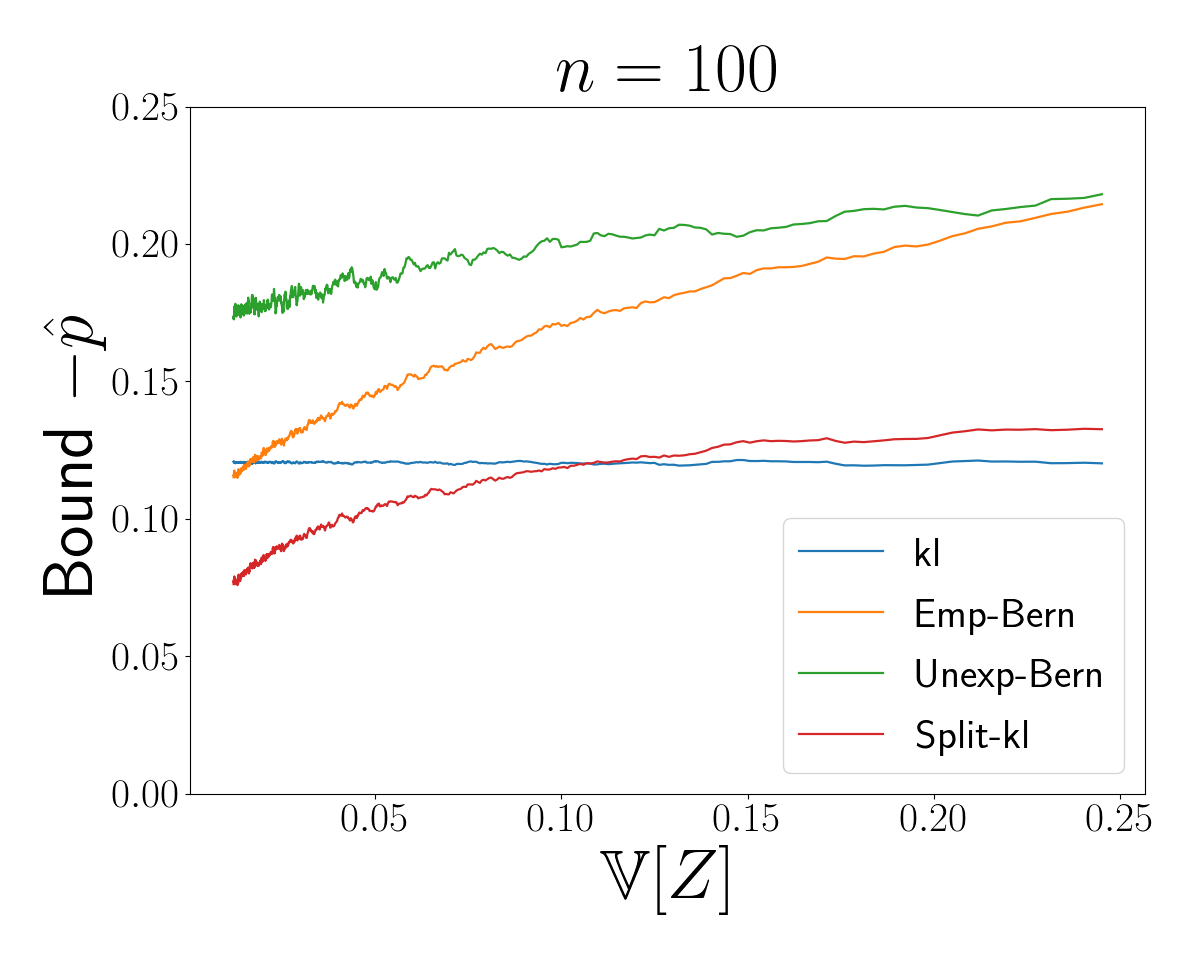}
		\caption{$n=100$.}
		\label{fig:app:sim:n100_ConstM}
	\end{subfigure}
	\hfill
	\begin{subfigure}[b]{.48\textwidth}
		\includegraphics[width=\textwidth]{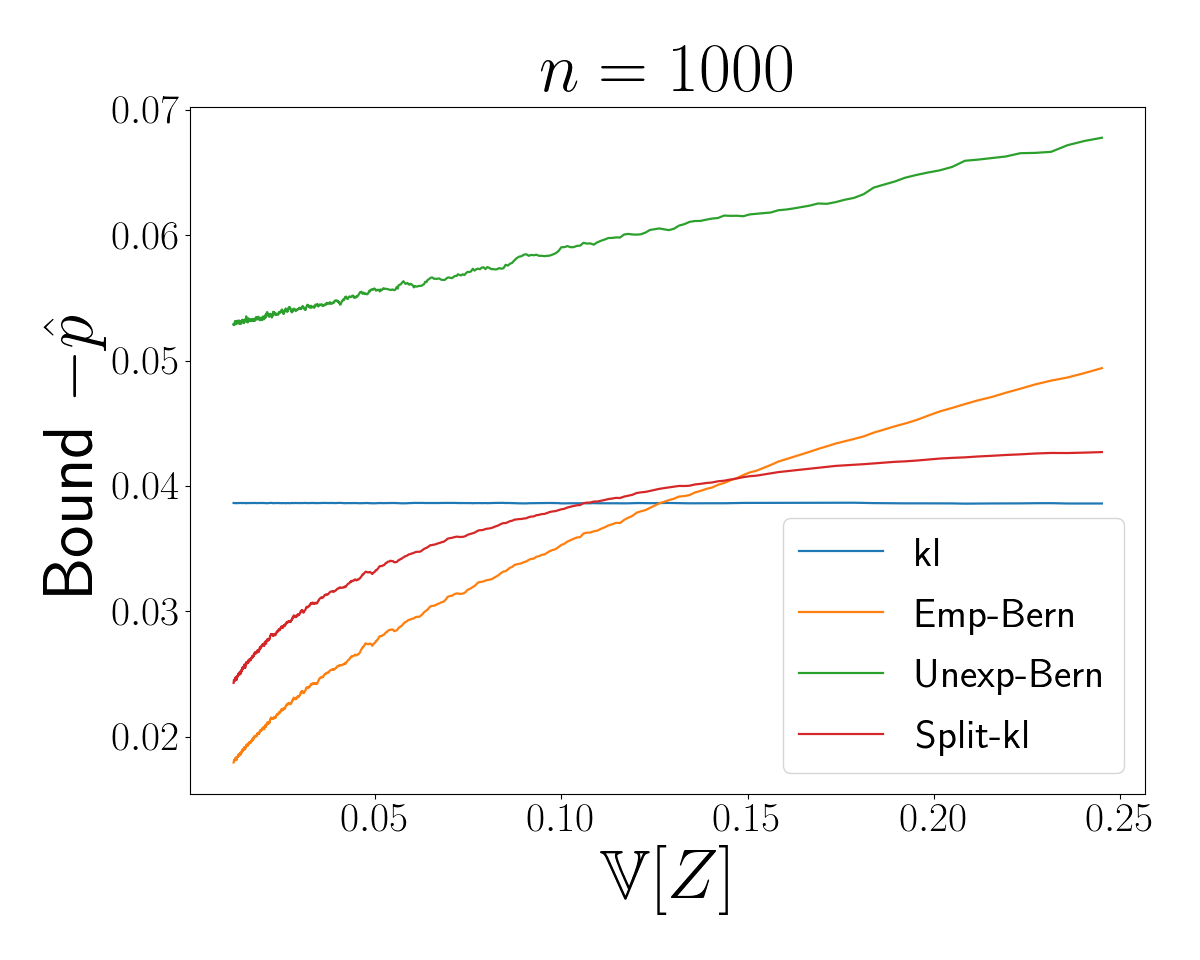}
		\caption{$n=1000$.}
		\label{fig:app:sim:n1000_ConstM}
	\end{subfigure}
	\caption{Comparison of the concentration bounds for beta distributions with parameters $\alpha=\beta$ taking values in the $[0.01,10]$ interval, with $\delta=0.05$, and with the number of samples $n=100$ and $n=1000$, respectively.}
	\label{fig:app:bounded:ConstM}
\end{figure}

In Figure~\ref{fig:app:bounded:spectrum} we consider another case where the variances stay similar but the means lie across the spectrum in between $0$ and $1$. We define the distributions being studied by a combination of two sets of probability distributions. First of all, we take $\beta=5$ and $\alpha\in[0.01, 5]$, resulting in the mean $p$ in between $0$ and $0.5$. We define another part by taking $\alpha=5$ and $\beta\in[0.01, 5]$, resulting in the mean $p$ in between $0.5$ and $1$. We plot the difference between the values of the bounds and $\hat p$ as a function of $p$. The $\kl$ bound is relatively weak around $p=0.5$ as expected. Since the variances stay similar across the interval, the performance of the Empirical Bernstein stay similar throughout the spectrum, and is tighter than the $\kl$ bound when the number of samples is sufficiently large. The split-$\kl$ bound is comparable and sometimes outperform the tightest bounds. The Unexpected Bernstein bound again falls behind due to the uncentered second moments.

\begin{figure}[t]
	\begin{subfigure}[b]{.48\textwidth}
		\includegraphics[width=\textwidth]{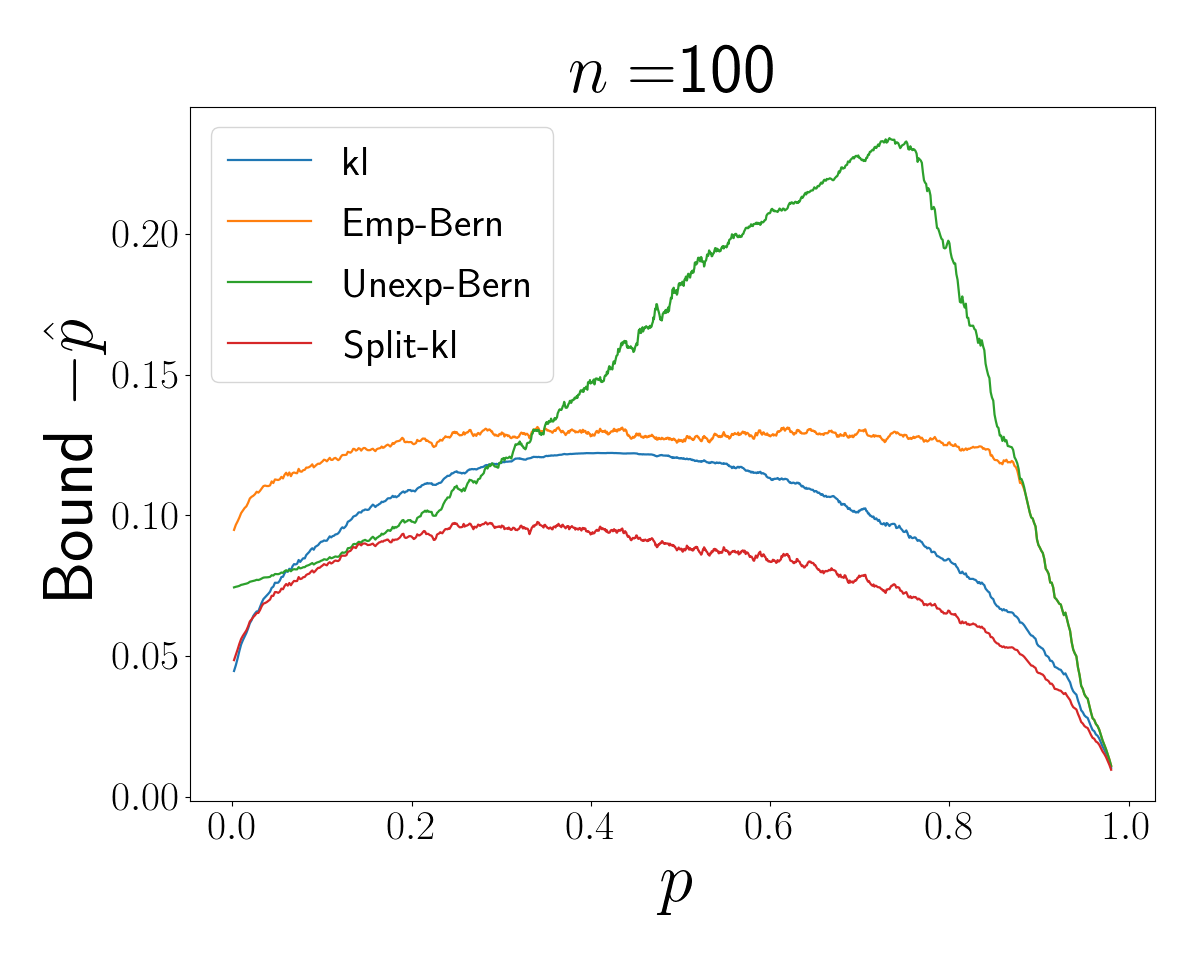}
		\caption{$n=100$.}
		\label{fig:app:sim:n100_spectrum}
	\end{subfigure}
	\hfill
	\begin{subfigure}[b]{.48\textwidth}
		\includegraphics[width=\textwidth]{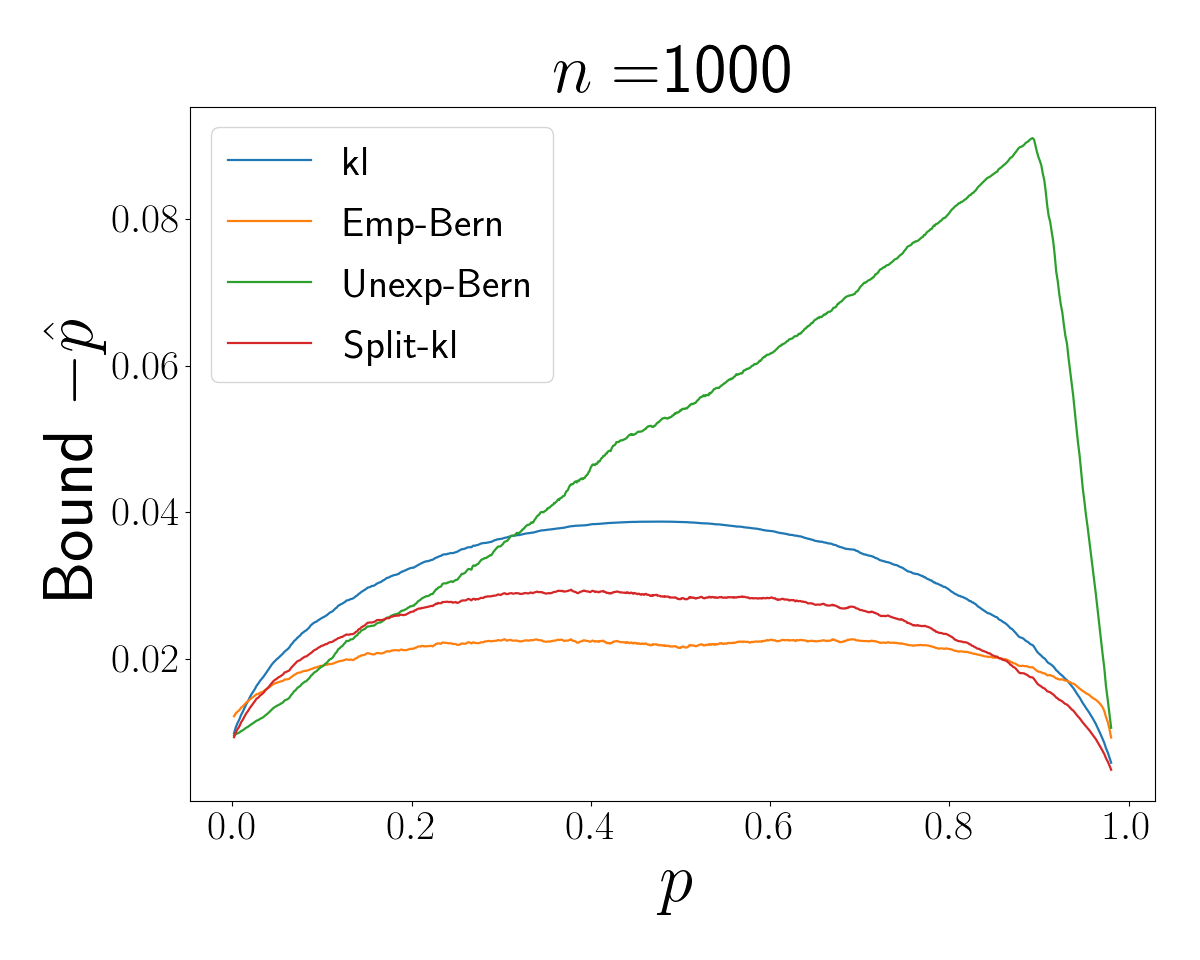}
		\caption{$n=100$.}
		\label{fig:app:sim:n1000_spectrum}
	\end{subfigure}
	\caption{Empirical comparison of concentration bounds for beta distribution with parameters $\alpha$ and $\beta$ and with the number of samples $n=100$ and $n=1000$. For $p\in[0,0.5]$, we take $\beta=5$ and $\alpha\in[0.01, 5]$ while for $p\in[0.5, 1]$, we take $\alpha=5$ and $\beta\in[0.01,5]$.}
	\label{fig:app:bounded:spectrum}
\end{figure}

\section{Experiments}\label{app:sec:experiments}

\subsection{Data Sets}\label{app:sec:datasets}
As mentioned in Section~\ref{sec:experiments}, we consider data sets from UCI and LibSVM repositories~\citep{UCI,libsvm}, as well as Fashion-MNIST from Zalando Research~\footnote{https://github.com/zalandoresearch/fashion-mnist}. An overview of the data sets is listed in Table~\ref{tab:data_sets}, where Banknote stands for Banknote Authentication, Breast-C stands for Breast Cancer Wisconsin, Fashion stands for Fashion-MNIST, Haberman stands for Haberman's Survival, and kr-vs-kp stands for Chess (King-Rook vs. King-Pawn). For data sets with a training and a testing set, we combine the training and the testing sets.
\begin{table}[t]
    \centering
    \caption{Data set overview. $c_{\min}$ and $c_{\max}$ denote the minimum and maximum class frequency.}
    \label{tab:data_sets}
    \begin{tabular}{lcccccc}\toprule
Data set & $N$ & $d$ & $c$ & $c_{\min}$ & $c_{\max}$ & Source\\
\midrule
\dataset{Adult} & 32561 & 14 & 2 & 0.2408 & 0.7592 & LIBSVM (a1a) \\
\dataset{Banknote} & 1372 & 4 & 2 & 0.4447 & 0.5553 & UCI \\
\dataset{Breast-C} & 699 & 9 & 2 & 0.3448 & 06552 & UCI \\
\dataset{Cod-RNA} & 59535 & 8 & 2 & 0.3333 & 0.6667 & LIBSVM \\
\dataset{Connect-4} & 67557 & 126 & 3 & 0.0955 & 0.6583 & LIBSVM \\
\dataset{Fashion} & 70000 & 784 & 10 & 0.1000 & 0.1000 & Zalando Research \\
\dataset{Haberman} & 306 & 3 & 2 & 0.2647 & 0.7353 & UCI \\
\dataset{kr-vs-kp} & 3196 & 36 & 2 & 0.48 & 0.52 & UCI \\
\dataset{Letter} & 20000 & 16 & 26 & 0.0367 & 0.0406 & UCI \\
\dataset{MNIST} & 70000 & 780 & 10 & 0.0902 & 0.1125 & LIBSVM \\
\dataset{Mushroom} & 8124 & 22 & 2 & 0.4820 & 0.5180 & LIBSVM \\
\dataset{Pendigits} & 10992 & 16 & 10 & 0.0960 & 0.1041 & LIBSVM \\
\dataset{Phishing} & 11055 & 68 & 2 & 0.4431 & 0.5569 & LIBSVM \\
\dataset{Protein} & 24387 & 357 & 3 & 0.2153 & 0.4638 & LIBSVM \\
\dataset{SVMGuide1} & 3089 & 4 & 2 & 0.3525 & 0.6475 & LIBSVM \\
\dataset{SatImage} & 6435 & 36 & 6 & 0.0973 & 0.2382 & LIBSVM \\
\dataset{Sensorless} & 58509 & 48 & 11 & 0.0909 & 0.0909 & LIBSVM \\
\dataset{Shuttle} & 58000 & 9 & 7 & 0.0002 & 0.7860 & LIBSVM \\
\dataset{Spambase} & 4601 & 57 & 2 & 0.394 & 0.606 & UCI \\
\dataset{Splice} & 3175 & 60 & 2 & 0.4809 & 0.5191 & LIBSVM \\
\dataset{TicTacToe} & 958 & 9 & 2 & 0.347 & 0.653 & UCI \\
\dataset{USPS} & 9298 & 256 & 10 & 0.0761 & 0.1670 & LIBSVM \\
\dataset{w1a} & 49749 & 300 & 2 & 0.0297 & 0.9703 & LIBSVM \\
\bottomrule
\end{tabular}

\end{table}

\paragraph{Linear Classifiers.} For the linear classifiers experiments, we consider selected data sets with binary class ($c=2$). We rescale all the real-valued attributes to the $[-1,1]$ interval and use one-hot encoding to encode categorical variables to $\{-1,1\}$, which increases the dimension of the attributes for some of the data sets. In particular, the effective dimension of Adult becomes 108, kr-vs-kp becomes 73, and Mushroom becomes 116. We remove rows containing missing features. For each data set, we shuffle the data sets and take four 5-fold train-test split, which gives 20 runs in total.

\paragraph{Weighted Majority Vote.} For the weighted majority vote experiments, including the ensemble of multiple heterogeneous classifiers and the random forest, we consider several binary and multiclass ($c>2$) data sets. We encode the categorical variables into integers and remove rows containing missing features. For each data set we take 10 runs, and for each run we randomly set aside 20\% of sample as the test set.

\subsection{Linear Classifiers}\label{app:sec:LC}
In this section, we describe the details of the experimental setting of the linear classifiers~\ref{app:sec:lc:setting}, the details of the bounds~\ref{app:sec:lc:bounds}, and the details of optimization~\ref{app:sec:lc:optimization}.  The  source code for replicating the experiments is available at Github\footnote{\url{https://github.com/YiShanAngWu/Split-KL-R}}.
\subsubsection{Experimental Setting}\label{app:sec:lc:setting}
In this section, we detail the settings and the construction of informed priors and excess losses using linear classifiers with Gaussian posterior. We follow the construction by~\citet{MGG20}.

As described in Section~\ref{sec:experiments}, the posterior $\rho=\mathcal{N}(w_S, \Sigma_S)$ is a Gaussian distribution centered at $w_S$, which is learned on $S$ using regularized logistic regression
\begin{equation}\label{eq:reg_logistic_regression}
    w_{S}=\argmin_{w\in\R^d} \frac{\lambda\|w\|^2}{2} + \frac{1}{|S|}\sum\limits_{(X,Y)\in S}-\lr{ Y \ln\phi(w^\top X)+(1-Y)\ln(1-\phi(w^\top X)) },
\end{equation}
where $\phi(x):=1/(1+e^{-x})$ for $x\in\R$ is the sigmoid function.
The covariance of the posterior is a diagonal matrix $\sigma^2 I_d$, where the variance $\sigma^2$ is learned from the corresponding PAC-Bayes bounds. We use the informed priors in all the PAC-Bayes bounds. The informed priors $\pi_{S_1}=\mathcal{N}(w_{S_1},\Sigma_{S_1})$ and $\pi_{S_2}=\mathcal{N}(w_{S_2},\Sigma_{S_2})$ are also chosen to be Gaussian distributions over $\R^d$, where the centers of the distributions are learned similarly using regularized logistic regression on the corresponding sample $S_1$ and $S_2$. If using excess losses, we take the classifier associated with $w_{S_1}$ as the reference classifier $h_{S_1}$ for the ``forward'' approach and take the classifier associated with $w_{S_2}$ as the reference classifier $h_{S_2}$ for the ``backward'' approach. We let the covariance of the informed priors to be also diagonal matrices $\Sigma_{S_1}=\Sigma_{S_2}=\sigma_\pi^2 I_d$, where $\sigma_\pi^2$ is selected from a grid $\mathcal{G}=\{1/2,\cdots,1/2^j\}$ for $j=\lceil\log_2 |S|\rceil$.

For all data sets, we use $\lambda=0.01$ in equation~\eqref{eq:reg_logistic_regression} and solve it using the BFGS algorithm. For all the bounds, we take $\delta=0.05$. Note that to be able to select the variance of the priors from a grid $\mathcal{G}$, we have to take a union bound over $\mathcal{G}$. Since the hypothesis space is infinitely large, we approximate the excess risk by drawing 100 classifiers from the posterior $\rho$ and compute the excess losses with respect to the reference classifiers. 

\subsubsection{Bounds}\label{app:sec:lc:bounds}
As mentioned in the body that we used informed priors for all the bounds we applied. The $\PBSkl_{\Ex}$ bound is presented in Theorem~\ref{thm:together_refined}, while the $\PBUB_{\Ex}$ bound and the $\PBkl$ bound will be presented in the following. The idea to derive PAC-Bayes bounds with informed priors in general is similar to the technique used in the proof of Theorem~\ref{thm:together_refined} in Appendix~\ref{sec:pf_thm:together_refined}.

The key element of the derivations is to bound $\E_{S'}[\E_\pi[e^{f_n(h,S')}]]$ for a given function $f_n:\mathcal{H}\times\lr{\mathcal{X}\times\mathcal{Y}}^n\rightarrow \R$ in Lemma~\ref{lem:PAC-Bayes}. Let the prior $\pi=\frac{1}{2}\pi_{S_1}+\frac{1}{2}\pi_{S_2}$, and let $S_*$ be either $S_1$ or $S_2$. If $h$ is sampled from $\pi_{S_*}$, we estimate the loss on $\bar S_*=S\backslash S_*$. Then,
\begin{align*}
    \E_{S}\E_{\pi}\lrs{e^{f_n(h,S)}} = \frac{1}{2} \sum\limits_{i=1,2} \E_S \E_{\pi_{S_i}} \lrs{e^{f_n(h,S)}}=\frac{1}{2} \sum\limits_{i=1,2} \E_{S_i} \E_{\pi_{S_i}} \E_{\bar{S}_i} \lrs{e^{f_n(h,S)}},
\end{align*}
where the second equality is due to the fact that $\pi_{S_*}$ is independent of $\bar{S}_*$ so they are exchangeable. We will then select the function $f_n(h,S)$ later such that $\E_{S_i}\E_{\pi_{S_i}}\E_{\bar S_i}[e^{f_n(h,S)}]$ is bounded for $i=1,2$.

Similarly, we let $\rho=\frac{1}{2}\rho_1+\frac{1}{2}\rho_2$. If $h$ is sampled from $\rho_*$, we estimate the loss on $\bar S_*=S\backslash S_*$. Then we have
\begin{equation}\label{eq:rewrite-IP-true}
\E_\rho[\tilde L(h)] = \frac{1}{2}\E_{\rho_1}[\tilde L(h)] + \frac{1}{2}\E_{\rho_2}[\tilde L(h)]
\end{equation}
and
\begin{equation}\label{eq:rewrite-IP-empirical}
    \E_\rho[\hat{\tilde L}(h,S_*)] = \frac{1}{2}\E_{\rho_1}[\hat{\tilde L}(h,S_2)] + \frac{1}{2}\E_{\rho_2}[\hat{\tilde L}(h,S_1)]
\end{equation}
for any loss $\tilde \ell$ and the corresponding quantities following the definitions in Section~\ref{sec:PBSkl}.
We assume that $\rho_1=\rho_2=\rho$ in all the bounds. Note that for simpler computation, we replace $\kl(\rho\|\pi)$ by its upper bound $\frac{1}{2}\kl(\rho\|\pi_{S_1})+\frac{1}{2}\kl(\rho\|\pi_{S_2})$ for all the bounds in the experiments.

\paragraph{PAC-Bayes-$\kl$ bound with Informed Priors ($\PBkl$).}
We take the PAC-Bayes-$\kl$ bound with informed priors as the baseline:
\begin{equation*}
    \E_\rho[L(h)]\leq \kl^{-1,+}\lr{\frac{1}{2}\E_\rho[\hat L(h,S_1)]+\frac{1}{2}\E_\rho[\hat L(h,S_2)],\frac{\KL(\rho\|\pi)+\ln\frac{2|\mathcal{G}|\sqrt{n/2}}{\delta}}{n/2}},
\end{equation*}
which is obtained by letting $f_n(h,S)=\frac{n}{2}\kl(\hat L(h,\bar S_*)\| L(h))$ and plugging it into Lemma~\ref{lem:PAC-Bayes}. In particular, we have $\E_{S_i}\E_{\pi_{S_i}}\E_{\bar S_i}[e^{f_n(h,S)}]=\E_{S_i}\E_{\pi_{S_i}}\E_{\bar S_i}[e^{\frac{n}{2}\kl(\hat L(h,\bar S_i)\| L(h))}]\leq 2\sqrt{n/2}$ for $i=1,2$ by Lemma~\ref{lem:Note-PB-Maurer}. Also, by the convexity of $\KL$, we further have 
\[
\kl\lr{\E_\rho[\hat{L}(h,S_*)]\|\E_\rho[L(h)]} \leq \E_\rho\lrs{\kl(\hat L(h,S_*)\| L(h))}.
\]
By taking the inverse of $\kl$, applying the relations in Eq.~\eqref{eq:rewrite-IP-true} and Eq.~\eqref{eq:rewrite-IP-empirical}, and taking a union bound over $\mathcal{G}$, we obtain the desired formula.

\paragraph{PAC-Bayes-Unexpected Bernstein Bound with Excess Loss and Informed Priors ($\PBUB_{\Ex}$).}
Let $\Delta_{\hat{\V}}(h,h^*,S)=\frac{1}{|S|}\sum_{(X,Y)\in S}( \Delta_\ell(h(X),h^*(X),Y))^2$ denote the average of the second moment of the excess losses. Then, the $\PBUB_{\Ex}$ has the form:
\begin{multline*}
    \E_\rho[L(h)]\leq \frac{1}{2}\E_\rho[\Delta_{\hat L}(h,h_{S_1},S_2)] + \frac{1}{2}\E_\rho[\Delta_{\hat L}(h,h_{S_2},S_1)]\nonumber\\
    + \frac{\psi(-\gamma b)}{\gamma b^2}\lr{\frac{1}{2}\E_\rho[\Delta_{\hat{\V}}(h,h_{S_1},S_2)]+\frac{1}{2}\E_\rho[\Delta_{\hat{\V}}(h,h_{S_2},S_1)]} + \frac{\KL(\rho\|\pi)+\ln\frac{3|\mathcal{G}||\Gamma|}{\delta}}{\gamma (n/2)}\nonumber\\
     +\Bin^{-1}\lr{\frac{n}{2}, \frac{n}{2}\hat{L}(h_{S_1}, S_2), \frac{\delta}{3|\mathcal{G}|}} + \Bin^{-1}\lr{\frac{n}{2}, \frac{n}{2}\hat{L}(h_{S_2}, S_1), \frac{\delta}{3|\mathcal{G}|}},
\end{multline*}
where $|\Gamma|$ comes from a union bound over a grid of $\gamma \in \Gamma=\{1/(2b), \cdots, 1/(2^k b)\}$ for $k=\lceil \log_2(\sqrt{|S|/\ln(1/\delta)}/2) \rceil$ when applying the PAC-Bayes-Unexpected-Bernstein inequality.

The last line of the bound is by applying Theorem~\ref{thm:test-set-bound} to $L(h_{S_1})$ and $L(h_{S_2})$ as in Theorem~\ref{thm:together_refined}, while the first two lines of the bound are derived from applying the PAC-Bayes-Unexpected-Bernstein inequality to the first two terms in equation~\eqref{eq:excess&inform}. In particular, let $f_n(h,S)=\gamma \frac{n}{2}\lr{\Delta_L(h,h_{S_*})-\Delta_{\hat L}(h,h_{S_*},\bar S_*)}-\frac{\psi(-b\gamma)}{b^2}\frac{n}{2}\Delta_{\hat \V}(h,h_{S_*},\bar S_*)$ and plug it into Lemma~\ref{lem:PAC-Bayes}. Then we have \[\E_{S_i}\E_{\pi_{S_i}}\E_{\bar S_i}[e^{f_n(h,S)}]=\E_{S_i}\E_{\pi_{S_i}}\E_{\bar S_i}[e^{\gamma \frac{n}{2}\lr{\Delta_L(h,h_{S_i})-\Delta_{\hat L}(h,h_{S_i},\bar S_i)}-\frac{\psi(-b\gamma)}{b^2}\frac{n}{2}\Delta_{\hat \V}(h,h_{S_i},\bar S_i)}]\leq 1\]
for $i=1,2$ by Lemma~\ref{lem:unepected-bernstein}. By moving the empirical quantities to the right hand side, applying the relations in Eq.~\eqref{eq:rewrite-IP-true} and Eq.~\eqref{eq:rewrite-IP-empirical}, and taking the union bounds, we obtain the desired formula.

\paragraph{PAC-Bayes-spli-$\kl$ Bound with Excess Loss and Informed Priors ($\PBSkl_{\Ex}$).}
The bound is stated in Theorem~\ref{thm:together_refined}, except that we replace $\delta$ by $\delta/|\mathcal{G}|$ for the union bound of $\mathcal{G}$. We take $\mu=0$ for the bound in the experiments.

\subsubsection{Optimization}\label{app:sec:lc:optimization}
Since the center of the posterior $w_S$ is learned using regularized logistic regression, the only thing remains is to decide the variance of the posterior $\sigma^2$ using the PAC-Bayes bounds. In general, the variance can be any non-negative values since the bound holds with high probability for all $\rho$ simultaneously. For simpler computation, we only consider the variance taking the same value as the variance of the priors \ie taking $\sigma^2=\sigma_\pi^2\in\mathcal{G}$. For each PAC-Bayes bounds, we find the optimal $\sigma^2$ by iterating over variances $\sigma^2=\sigma_\pi^2\in\mathcal{G}$ and return the one corresponds to the tightest bound. We approximate $\E_\rho[\cdot]$ by sampling $100$ classifiers from $\rho$.

The inverse $\kl$ in the $\PBkl$ and the $\PBSkl_{\Ex}$ bounds can be computed by binary search. The inverse of the binomial tail distribution in the $\PBUB_{\Ex}$ and the $\PBSkl_{\Ex}$ bounds can also be computed by binary search. To optimize the $\PBUB_{\Ex}$ bound, we also need to iterate over $\gamma\in\Gamma$.

\subsection{Ensemble of Multiple Heterogeneous Classifiers}\label{app:sec:MCE}
In this section, we describe the details of the experimental setting of the ensemble of multiple heterogeneous classifiers~\ref{app:sec:mce:setting}, the details of bounds and optimization~\ref{app:sec:mce:bounds&optimization}, and lastly, the results~\ref{app:sec:mce:results}.  The  source code for replicating the experiments is available at Github\footnote{\url{https://github.com/StephanLorenzen/MajorityVoteBounds}}.

\subsubsection{Experimental Setting}\label{app:sec:mce:setting}
In this experiment, we follow the setting in~\citet{WMLIS21}. We take the following standard classifiers available in \textit{scikit-learn} using default parameters to build the ensemble: 1. \textbf{Linear Discriminant Analysis} 2. \textbf{Decision Tree} 3. \textbf{Logistic Regression} 4. \textbf{Gaussian Naive Bayes}. We also take three versions of \textbf{k-Nearest Neighbors}: 1. $k=3$ with uniform weights (\ie all points in each neighborhood are weighted equally) 2. $k=5$ with uniform weights, and 3. $k=5$ with the weights of the points are defined by the inverse of their L2 distance. Thus, there are 7 classifiers for ensemble in total. 

\paragraph{Ensemble Construction by Bagging.} We follow the construction used by~\citet{MLIS20,WMLIS21}. For each classifier $h$, we generate a random split of the data set $S$ into a pair of subsets $S=S_h\cup \bar S_h$, where $\Bar S_h=S\backslash S_h$. We generate the split by the standard bagging method, where $S_h$ contains $0.8|S|$ samples randomly subsampled with replacement from $S$.  We train the classifier on $S_h$, and estimate the expected loss on the out-of-bag (OOB) sample $\bar S_h$ to make an unbiased estimation. The resulting set of classifiers produces an ensemble, while the estimates are used for calculating the bounds and deciding the weights of the ensemble. In particular, we estimate the expected loss by $\hat L(h,\bar S_h)$, and let $n=\min_h|\bar S_h|$. In the remaining of the paper, we call the tandem loss with an offset $\alpha$ by the $\alpha$-tandem loss. Then for a pair of classifiers $h$ and $h'$, we take the overlap of the OOB sample $\bar S_h\cap \bar S_{h'}$ to estimate the unbiased tandem loss $\hat L(h,h',\bar S_h\cap \bar S_{h'})$, $\alpha$-tandem loss $\hat L_\alpha(h,h',\bar S_h\cap \bar S_{h'})$, the second moment of the $\alpha$-tandem loss $\hat \Var_\alpha(h,h',\bar S_h\cap \bar S_{h'})$, the variance of the $\alpha$-tandem loss $\hat{\operatorname{Var}}_\alpha(h,h',\bar S_h\cap \bar S_{h'})$, as well as the splits of the $\alpha$-tandem loss $\hat L^+_\alpha(h,h',\bar S_h\cap \bar S_{h'})$ and $\hat L^-_\alpha(h,h',\bar S_h\cap \bar S_{h'})$. Let $m=\min_{h,h'}|\bar S_h\cap \bar S_{h'}|$ be the minimum size of the overlap.

\subsubsection{Bounds and Optimization}\label{app:sec:mce:bounds&optimization} The bounds we are comparing in this section are the $\TND$, $\CCTND$, $\CCPBB$, $\CCPBUB$, and $\CCPBSkl$ bounds. The derivations of the first three bounds are provided in~\citet{MLIS20,WMLIS21}, while we will provide the derivations of the $\CCPBUB$ bound and the $\CCPBSkl$ bound.

In the experiments, we take $\delta=0.05$ and take $\pi$ to be a uniform distribution over the classifiers. For $\CCPBB,\CCPBUB$ and $\CCPBSkl$ bounds, we take a grid of $\alpha\in[-0.5,0.5]$ since the bounds are not differentiable w.r.t\ $\alpha$. Note that we don't need a union bound over $\alpha$~\citep{WMLIS21}. To optimize the weighting $\rho$, we
applied iRProp+ for the gradient based optimization \citep{IH03,FI18}, until
the bound did not improve more than $10$ for 10 iterations. To find the optimal $\rho$ and the parameters, we start by $\rho=\pi$, and apply alternating minimization until the bound doesn't change for more than $10^{-9}$. The details of alternating minimization for each bound are provided below.

We first cite the three existing bounds.
\paragraph{Tandem Bound ($\TND$)~\citep{MLIS20}}
They used the following formula to compute the bound after obtaining the optimal weights $\rho$:
\begin{equation*}\label{eq:TND-bound}
    L(\MV_\rho) \leq 4\kl^{-1,+}\lr{\E_{\rho^2}[\hat L(h,h',\bar S_h\cap \bar S_{h'})], \frac{2\KL(\rho\|\pi)+\ln(4\sqrt{m}/\delta)}{m}},
\end{equation*}
and used the following relaxation, based on the PAC-Bayes-$\lambda$ inequality~\ref{app:sec:PBlambda}, for easier optimization:
\begin{equation}\label{eq:TND-optimization}
    L(\MV_\rho) \leq 4\lr{\frac{\E_{\rho^2}[\hat L(h,h',\bar S_h\cap \bar S_{h'})]}{1-\lambda/2}+\frac{2\KL(\rho\|\pi)+\ln(2\sqrt{m}/\delta)}{\lambda(1-\lambda/2)m}}
\end{equation}
for any $\lambda\in(0,2)$. The bound can be optimized by implementing alternating minimization: Given $\rho$, starting with $\rho=\pi$, find the corresponding optimal $\lambda$ (Sec.~\ref{app:sec:PBlambda}). Then given $\lambda$, optimize $\rho$ by projected gradient descent.

\paragraph{Chebyshev-Cantelli bound with $\TND$ empirical loss estimate bound ($\CCTND$)~\citep{WMLIS21}} They used the following formula to compute the bound after obtaining the optimal weights $\rho$:
\begin{multline*}
L(\MV_\rho)\leq \frac{1}{(0.5-\alpha)^2}\bigg[\kl^{-1,+}\lr{\E_{\rho^2}[\hat L(h,h',\bar S_h\cap \bar S_{h'})], \frac{2\KL(\rho\|\pi)+\ln(4\sqrt{m}/\delta)}{m}}\\ 
- 2\alpha\kl^{-1,\circ}\lr{\E_\rho[\hat L(h,\bar S_h)] , \frac{\KL(\rho\|\pi) + \ln(4 \sqrt n/\delta)}{n}} + \alpha^2\bigg],
\end{multline*}
for $\alpha<0.5$, where $\circ$ is ``$-$'' for $\alpha\geq 0$ and ``$+$'' otherwise.
On the other hand, they used the following relaxations, based on the PAC-Bayes-$\lambda$ inequality~\ref{app:sec:PBlambda}, for easier optimization:
\begin{multline*}
L(\MV_\rho)\leq \frac{1}{(0.5-\alpha)^2}\bigg[\frac{\E_{\rho^2}[\hat L(h,h',\bar S_h\cap \bar S_{h'})]}{1 - \frac{\lambda}{2}} + \frac{2\KL(\rho\|\pi) + \ln(4 \sqrt m/\delta)}{\lambda\lr{1-\frac{\lambda}{2}}m}\\ 
- 2\alpha\lr{\lr{1 - \frac{\gamma}{2}}\E_\rho[\hat L(h,\bar S_h)] - \frac{\KL(\rho\|\pi) + \ln(4 \sqrt n/\delta)}{\gamma n}} + \alpha^2\bigg]
\end{multline*}
for $0\leq\alpha<0.5$, and
\begin{multline*}
L(\MV_\rho)\leq \frac{1}{(0.5-\alpha)^2}\bigg[\frac{\E_{\rho^2}[\hat L(h,h',\bar S_h\cap \bar S_{h'})]}{1 - \frac{\lambda}{2}} + \frac{2\KL(\rho\|\pi) + \ln(4 \sqrt m/\delta)}{\lambda\lr{1-\frac{\lambda}{2}}m}\\ 
- 2\alpha\lr{ \frac{\E_{\rho}[\hat L(h,\bar S_h)]}{1 - \frac{\gamma}{2}} + \frac{\KL(\rho\|\pi) + \ln(4\sqrt n/\delta)}{\gamma\lr{1-\frac{\gamma}{2}}n}} + \alpha^2\bigg]
\end{multline*}
for $\alpha<0$. The optimization of the bound can, again, be done by alternating minimization of the following steps: 1. Given $\alpha$ and $\rho$, where we start with $\alpha=0$ and $\rho=\pi$, compute the corresponding closed-form minimizer $\lambda$ and $\gamma$ (Sec.~\ref{app:sec:PBlambda}). 2. Given $\rho$, $\lambda$ and $\gamma$, find the closed-form minimizer $\alpha$. 3. Given parameters $\alpha$, $\lambda$, and $\gamma$, optimize over $\rho$ using projected gradient descent.

\paragraph{Chebyshev-Cantelli bound with PAC-Bayes-Bennett loss estimate bound ($\CCPBB$)~\citep{WMLIS21}}
The $\CCPBB$ bound has the following formula for both computing the bound and optimization:
\begin{multline*}
L(\MV_\rho) \leq \frac{1}{(0.5-\alpha)^2}\Bigg[\E_{\rho^2}[\hat L_\alpha(h,h',\bar S_h\cap \bar S_{h'})] + \frac{2\KL(\rho\|\pi) + \ln\frac{2k_\lambda k_\gamma}{\delta}}{\gamma m}\\
+\frac{\phi(\gamma K_\alpha)}{\gamma K_\alpha^2}\lr{\frac{\E_{\rho^2}[\hat{\operatorname{Var}}_\alpha(h,h',\bar S_h\cap \bar S_{h'})]}{1-\frac{\lambda m}{2(m-1)}} + \frac{K_\alpha^2\lr{2\KL(\rho\|\pi)+\ln \frac{2k_\lambda k_\gamma}{\delta}}}{n\lambda \lr{1-\frac{\lambda m}{2(m-1)}}}}\Bigg],
\end{multline*}
where $\phi(x)=e^x-x-1$ and $K_\alpha=\max\{1-\alpha,1-2\alpha\}$ is the length of the range of the $\alpha$-tandem loss. The parameter $\gamma$ is taken in a grid $\lrc{\gamma_1,\cdots,\gamma_{k_\gamma}}$, where $\gamma_i>0$ for all $i$ and $\lambda$ is taken in a grid $\lrc{\lambda_1,\cdots,\lambda_{k_\lambda}}$, where $\lambda_i\in\lr{0,\frac{2(n-1)}{n}}$ for all $i$. $k_\gamma$ and $k_\lambda$ in the bound come from the union bounds over a grid of $\gamma$ and a grid of $\lambda$.


To optimize the bound, we take a grid of $\alpha\in[-0.5, 0.5]$ and iterate over $\alpha$ in the grid. For a given $\alpha$, we first compute $\hat L_\alpha(h,h',\bar S_h\cap \bar S_{h'})$ and $\hat{\operatorname{Var}}_\alpha(h,h',\bar S_h\cap \bar S_{h'})$ for all $h,h'$. Then, optimize the bound for a fix $\alpha$ by alternating the following steps: 1. Given $\rho$, starting with $\rho=\pi$, find the corresponding optimal $\lambda$, and then the optimal $\gamma$ in the grids. 2. Given $\lambda$ and $\gamma$, optimize $\rho$ by projected gradient descent.

Next, we present the two new bounds $\CCPBUB$ and $\CCPBSkl$, which are based on the oracle parametric form of the Chebyshev-Cantelli bound~\citep[Theorem 8]{WMLIS21}:
For all $\rho$ and for all $\alpha<0.5$
\begin{equation}\label{eq:oracle_CC}
    L(\MV_\rho)\leq \frac{\E_{\rho^2}[L_\alpha(h,h')]}{(0.5-\alpha)^2}.
\end{equation}
By applying the PAC-Bayes-Unexpected-Bernstein inequality to the $\alpha$-tandem loss, we obtain the $\CCPBUB$ bound, while by applying the PAC-Bayes-split-$\kl$ inequality to the $\alpha$-tandem loss, we obtain the $\CCPBSkl$ bound.

\paragraph{Chebyshev-Cantelli bound with PAC-Bayes-Unexpected-Bernstein loss estimate bound ($\CCPBUB$)}
By applying the PAC-Bayes-Unexpected-Bernstein inequality to the $\alpha$-tandem loss in equation~\eqref{eq:oracle_CC}, with the upper bound of the $\alpha$-tandem loss $b=(1-\alpha)^2$ for $\alpha<0.5$, we obtain the bound:
\begin{multline*}
    L(\MV_\rho) \leq \frac{1}{(0.5-\alpha)^2}\bigg[\E_{\rho^2}[\hat L_\alpha(h,h',\bar S_h\cap \bar S_{h'})]+\frac{\psi(-\gamma (1-\alpha)^2)}{\gamma (1-\alpha)^4}\E_{\rho^2}[\hat{\V}_\alpha(h,h',\bar S_h\cap \bar S_{h'})]\\
    +\frac{2\KL(\rho\|\pi)+\ln \frac{k_\gamma}{\delta}}{\gamma m}\bigg],
\end{multline*}
where the $2$ in front of $\KL$ comes from the fact that $\KL(\rho^2\|\pi^2)=2\KL(\rho\|\pi)$. As in the previous experiments, we take a grid of $\gamma\in\{1/(2(1-\alpha)^2), \cdots,1/(2^{k_\gamma}(1-\alpha)^2)\}$ for $k_\gamma=\lceil\log_2(\sqrt{m/\ln(1/\delta)}/2) \rceil$ when applying the PAC-Bayes-Unexpected-Bernstein inequality.

Similar to the optimization of the $\CCPBB$ bound, we again take a grid of $\alpha\in[-0.5,0.5]$ and iterate over $\alpha$ in the grid. For a given $\alpha$, we first compute $\hat L_\alpha(h,h',\bar S_h\cap \bar S_{h'})$ and $\hat{\V}_\alpha(h,h',\bar S_h\cap \bar S_{h'})$ for all $h,h'$. Then, optimize the bound for a fix $\alpha$ by alternating minimization of $\rho$ and $\gamma$ in the grid. We initialize $\rho=\pi$ and optimize it by pojected gradient descent.

\paragraph{Chebyshev-Cantelli bound with PAC-Bayes-split-$\kl$ loss estimate bound ($\CCPBSkl$)}
Similarly, by applying the PAC-Bayes-split-$\kl$ inequality to the $\alpha$-tandem loss in equation~\eqref{eq:oracle_CC}, we obtain the following formula to compute the bound after obtaining the optimal weights $\rho$:
\begin{multline*}
    L(\MV_\rho) \leq \frac{1}{(0.5-\alpha)^2}\bigg[\mu+ \lr{b-\mu}\kl^{-1,+}\lr{\frac{\E_{\rho^2}[\hat L^+_\alpha(h,h',\bar S_h\cap \bar S_{h'})]}{b-\mu},\frac{2\KL(\rho\|\pi)+\ln\frac{4\sqrt{m}}{\delta}}{m}} \\ 
    - (\mu-a)\kl^{-1,-}\lr{\frac{\E_{\rho^2}[\hat L^-_\alpha(h,h',\bar S_h\cap \bar S_{h'})]}{\mu-a},\frac{2\KL(\rho\|\pi)+\ln\frac{4\sqrt{m}}{\delta}}{m}}\bigg].
\end{multline*}
We use the following relaxation formula, which is based on the PAC-Bayes-$\lambda$ inequality~\ref{app:sec:PBlambda}, for optimization.
\begin{multline*}
    L(\MV_\rho) \leq \frac{1}{(0.5-\alpha)^2}\bigg[\mu+ \lr{b-\mu}\lr{\frac{\E_{\rho^2}[\hat L^+_\alpha(h,h',\bar S_h\cap \bar S_{h'})]}{\lr{b-\mu}\lr{1-\lambda/2}}+\frac{2\KL(\rho\|\pi)+\ln\frac{4\sqrt{m}}{\delta}}{\lambda(1-\lambda/2)m}} \\ 
    - (\mu-a)\lr{\lr{1-\frac{\gamma}{2}}\frac{\E_{\rho^2}[\hat L^-_\alpha(h,h',\bar S_h\cap \bar S_{h'})]}{\mu-a}-\frac{2\KL(\rho\|\pi)+\ln\frac{4\sqrt{m}}{\delta}}{\gamma m}}\bigg].
\end{multline*}
$2\KL(\rho\|\pi)$ in both bounds again comes from $\KL(\rho^2\|\pi^2)=2\KL(\rho\|\pi)$.  Recall that the $\alpha$-tandem loss takes values in $\lrc{(1-\alpha)^2,-\alpha(1-\alpha),\alpha^2}$. For $\alpha<0.5$, $(1-\alpha)^2$ has the largest value.  Therefore, we take $b=(1-\alpha)^2$ in the bound. Furthermore, for $\alpha < 0$ we have $\alpha^2<-\alpha(1-\alpha)$, and for $\alpha \geq 0$ we have $-\alpha(1-\alpha)\leq \alpha^2$. Therefore, for $\alpha<0$, we take $a=\alpha^2$ and $\mu$ to be the middle value $-\alpha(1-\alpha)$, while for $\alpha\geq 0$, we take $a=-\alpha(1-\alpha)$ and $\mu=\alpha^2$.

To optimize the bound, we again take a grid of $\alpha\in[-0.5,0.5]$ and iterate over $\alpha$ in the grid. For a given $\alpha$, we compute the parameters $a, b, \mu$, and then compute the losses $\hat L^+_\alpha(h,h',\bar S_h\cap \bar S_{h'})$ and $\hat L^-_\alpha(h,h',\bar S_h\cap \bar S_{h'})$ for all $h,h'$. The optimization of the bound for a fixed $\alpha$ can be done by alternating minimization: 1. Given $\rho$, starting with $\rho=\pi$, compute the corresponding optimal $\lambda$ and $\gamma$ (Sec.~\ref{app:sec:PBlambda}). 2. Given $\lambda$ and $\gamma$, optimize over $\rho$ using projected gradient descent.

\subsubsection{Results}\label{app:sec:mce:results}
We presented in the body the results of the selected data sets. Here we show the results on more data sets. We present the results for binary data sets in Figure~\ref{fig:mce:binary}, while we present the results for multiclass data sets in Figure~\ref{fig:mce:multiclass}. Taking $\alpha=0$ for $\CCTND$, $\CCPBB$, $\CCPBUB$, and $\CCPBSkl$ bounds collapses to the $\TND$ bound. Therefore, we take $\TND$ as a baseline. In both figures, $\CCTND$ performs similar to, and often better than the baseline. The second bound, $\CCPBB$, using the $\alpha$-tandem loss, lags behind due to nested application of concentration bounds. The two new bounds based the $\alpha$-tandem loss, $\CCPBUB$ and $\CCPBSkl$, clearly improve the shortage of the $\CCPBB$ bound and often provide tighter bounds than the rest.
\begin{figure}[t]
    \centering
    \includegraphics[width=0.95\textwidth]{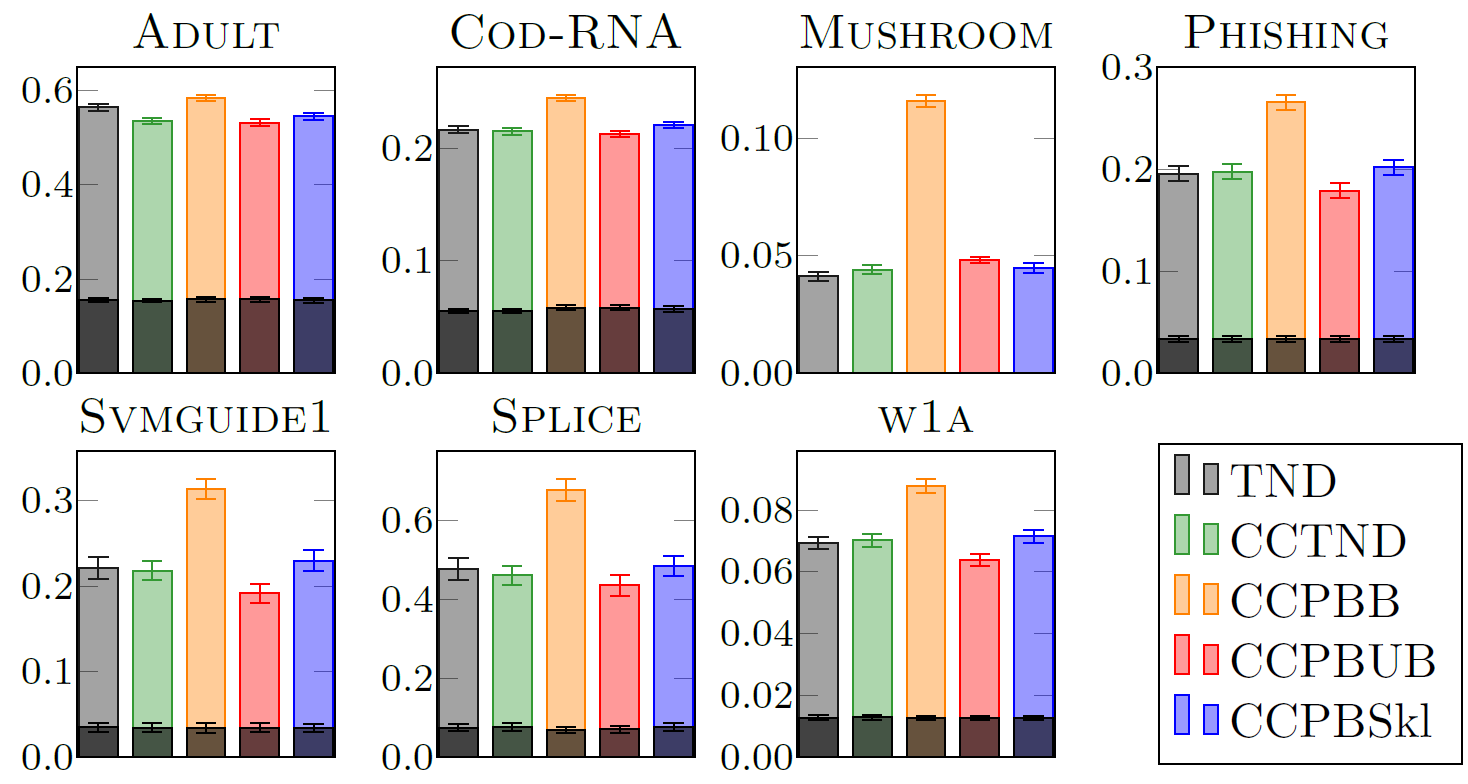}
    \caption{Comparison of the bounds and the test losses of the weighted majority vote on ensembles of heterogeneous classifiers with optimized posterior $\rho^*$ generated by $\TND$, $\CCTND$, $\CCPBB$, $\CCPBUB$, and $\CCPBSkl$. The data sets are binary labeled. The test losses of the corresponding bounds are shown in black. We report the mean and the standard deviation over 10 runs of the experiments.}
    \label{fig:mce:binary}
\end{figure}

\begin{figure}[t]
    \centering
    \includegraphics[width=0.95\textwidth]{experiment/plots/MCE-multiclass.png}
    \caption{Comparison of the bounds and the test losses of the weighted majority vote on ensembles of heterogeneous classifiers with optimized posterior $\rho^*$ generated by $\TND$, $\CCTND$, $\CCPBB$, $\CCPBUB$, and $\CCPBSkl$. The data sets are multiclass labeled. The test losses of the corresponding bounds are shown in black. We report the mean and the standard deviation over 10 runs of the experiments.}
    \label{fig:mce:multiclass}
\end{figure}

\subsection{Random Forest Majority Vote Classifiers}\label{app:sec:RFC}
In this section, we describe the details of the experimental setting of random forest~\ref{app:sec:rf:setting} and the results of the experiments~\ref{app:sec:rf:results}. Since both the ensemble of the heterogeneous classifiers and the random forest are examples of weighted majority vote, the bounds and optimization methods in this experiment are the same as described in Sec.~\ref{app:sec:mce:bounds&optimization}.  The  source code for replicating the experiments is available at Github\footnote{\url{https://github.com/StephanLorenzen/MajorityVoteBounds}}.
\subsubsection{Experimental Setting}\label{app:sec:rf:setting}
In this section, we follow the construction used by~\citet{WMLIS21}. 
We construct the ensemble from decision trees, which is available in \textit{scikit-learn}. We take 100 fully grown trees to build the random forest. The ensemble is again construct by bagging as described in~\ref{app:sec:mce:setting}, where each tree $h$ is trained on a subset of a random split $S_h$ and estimated on $\bar S_h$. To train each tree, we use the Gini criterion for splitting and consider $\sqrt{d}$ features in each split, where $d$ is the number of the attributes in data. 

\subsubsection{Results}\label{app:sec:rf:results}
The results of random forest weighted majority vote on binary data sets are shown in Figure~\ref{fig:rfc:binary} while the results on multiclass data sets are shown in Figure~\ref{fig:rfc:multiclass}. Similar to the discussions in Sec.~\ref{app:sec:mce:results} for the ensemble of heterogeneous classifiers, $\TND$ serves as a baseline. In both figures, $\CCTND$ performs similar to the baseline. The $\CCPBB$, using the $\alpha$-tandem loss, lags behind due to nested application of concentration bounds. The two new bounds based the $\alpha$-tandem loss, $\CCPBUB$ and $\CCPBSkl$, clearly improve the shortage of the $\CCPBB$ bound. The $\CCPBSkl$ bound is comparable to the baseline, and the $\CCPBUB$ bound often provide tighter bounds than the baseline.

\begin{figure}[t]
    \centering
    \includegraphics[width=\textwidth]{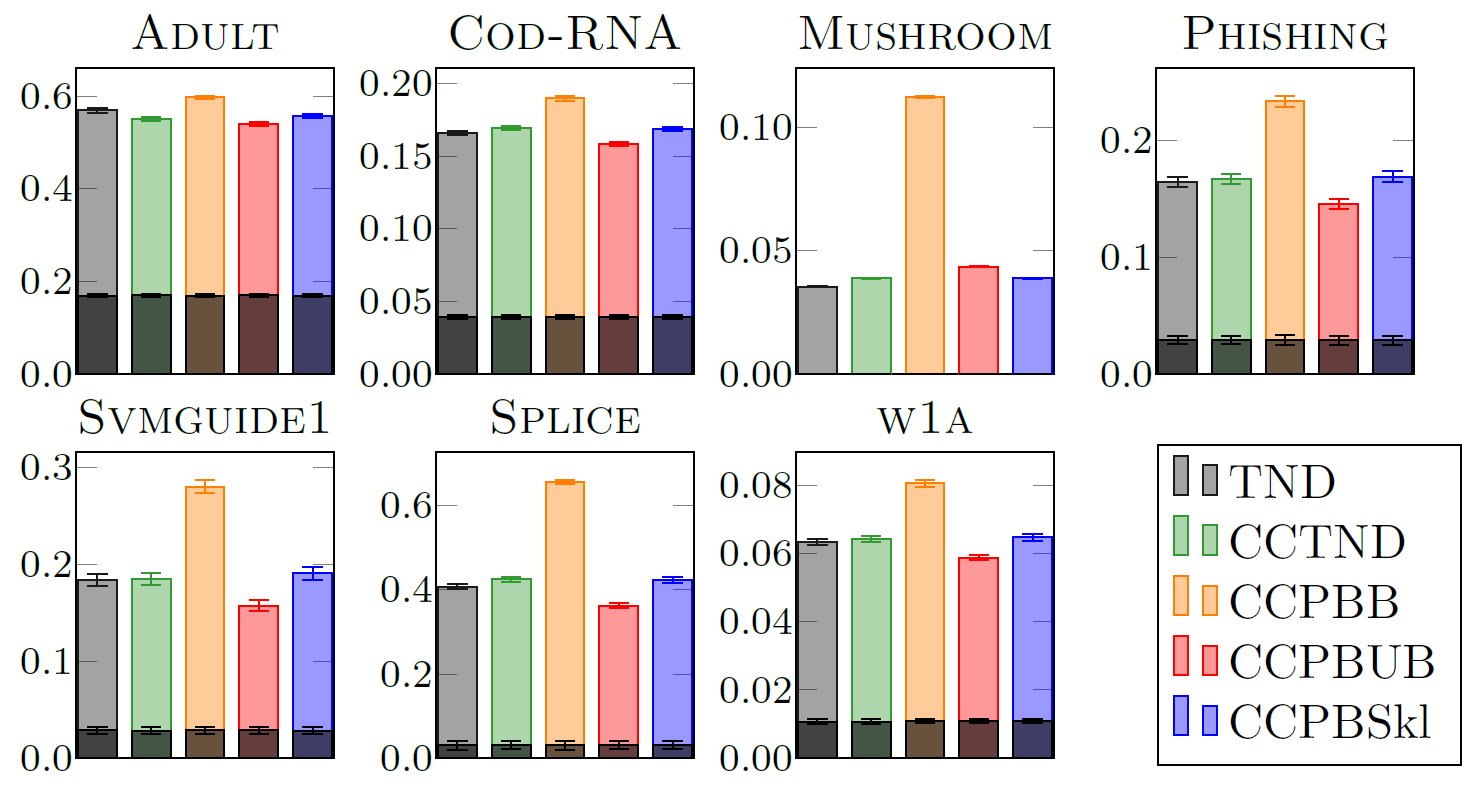}
    \caption{Comparison of the bounds and the test losses of the weighted majority vote on random forest with optimized posterior $\rho^*$ generated by $\TND$, $\CCTND$, $\CCPBB$, $\CCPBUB$, and $\CCPBSkl$. The data sets are binary labeled. The test losses of the corresponding bounds are shown in black. We report the mean and the standard deviation over 10 runs of the experiments.}
    \label{fig:rfc:binary}
\end{figure}

\begin{figure}[t]
    \centering
    \includegraphics[width=\textwidth]{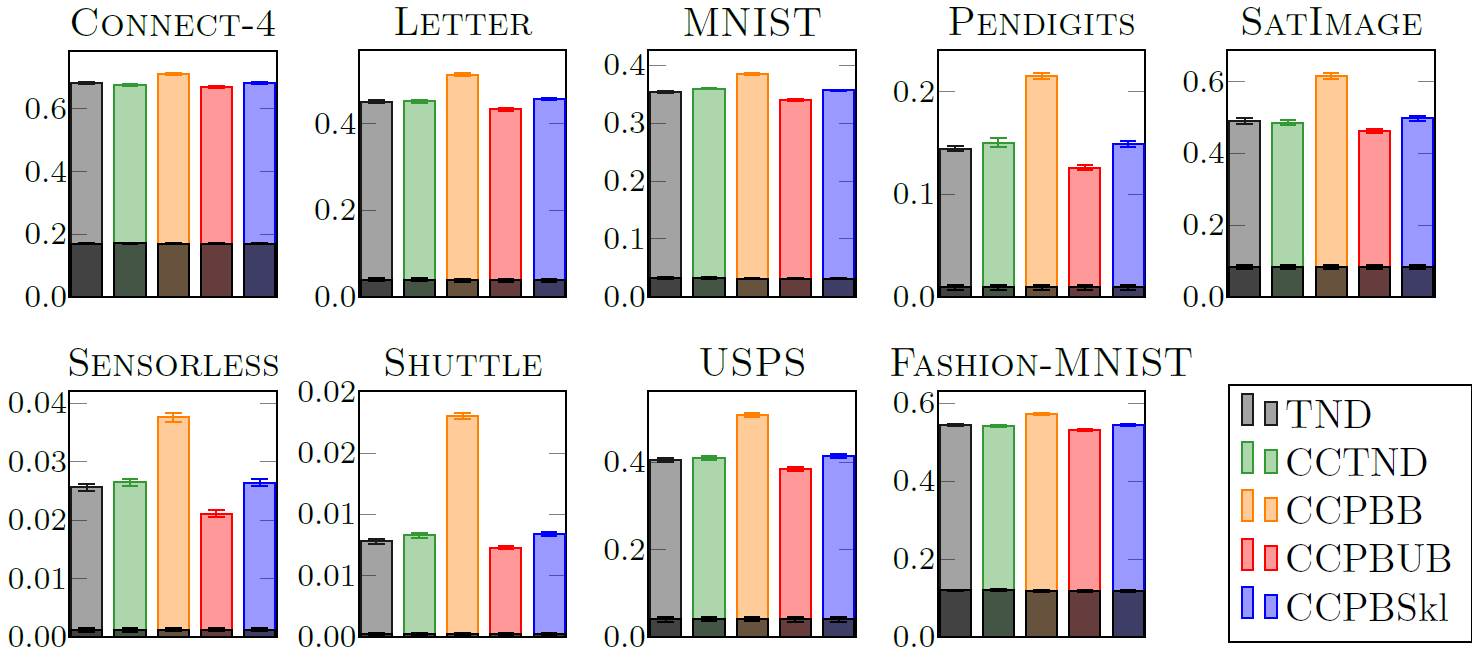}
    \caption{Comparison of the bounds and the test losses of the weighted majority vote on random forest with optimized posterior $\rho^*$ generated by $\TND$, $\CCTND$, $\CCPBB$, $\CCPBUB$, and $\CCPBSkl$. The data sets are multiclass labeled. The test losses of the corresponding bounds are shown in black. We report the mean and the standard deviation over 10 runs of the experiments.}
    \label{fig:rfc:multiclass}
\end{figure}

\section{PAC-Bayes-$\lambda$ Inequality}\label{app:sec:PBlambda}
\begin{theorem}[PAC-Bayes-$\lambda$ Inequality, \citealp{TIWS17,MLIS20}]\label{thm:lambdabound} For any loss $\ell\in[0,1]$, any probability distribution $\pi$ on ${\cal H}$ that is independent of $S$ and any $\delta \in (0,1)$, with probability at least $1-\delta$ over a random draw of a sample $S$, for all distributions $\rho$ on ${\cal H}$ and all $\lambda \in (0,2)$ and $\gamma > 0$ simultaneously:
\begin{align}
\E_\rho\lrs{L(h)} &\leq \frac{\E_\rho[\hat L(h,S)]}{1 - \frac{\lambda}{2}} + \frac{\KL(\rho\|\pi) + \ln(2 \sqrt n/\delta)}{\lambda\lr{1-\frac{\lambda}{2}}n},\label{eq:PBlambda}\\
\E_\rho\lrs{L(h)} &\geq \lr{1 - \frac{\gamma}{2}}\E_\rho[\hat L(h,S)] - \frac{\KL(\rho\|\pi) + \ln(2 \sqrt n/\delta)}{\gamma n}.\label{eq:PBlambda-lower}
\end{align}
\label{thm:PBlambda}
\end{theorem}
The upper bound is due to~\citet{TIWS17} and the lower bound is due to~\citet{MLIS20}, and both hold simultaneously. The PAC-Bayes-$\lambda$ bound is an optimization friendly relaxation of the PAC-Bayes-$\kl$ bound. ~\citet{TS13} has shown that for a given $\rho$, equation~\ref{eq:PBlambda} is convex in $\lambda$ and has the minimizer
\begin{equation*}\label{eq:opt_lambda}
    \lambda_\rho^* = \frac{2}{\sqrt{\frac{2n\E_{\rho}[\hat L(h,S)]}{\KL(\rho||\pi)+\ln \frac{2\sqrt{n}}{\delta}}+1}+1}.
\end{equation*}
On the other hand,~\citet{MLIS20} has shown that for a given $\rho$, the optimal $\gamma$ in equation~\ref{eq:PBlambda-lower} can be achieved by
\begin{equation*}\label{eq:opt_gamma}
    \gamma_\rho^* = \sqrt{\frac{\KL(\rho||\pi) + \ln(2\sqrt{n}/\delta)}{n\E_\rho[\hat L(h,S)] }}.
\end{equation*}
Furthermore, given $\lambda$ or $\gamma$, the optimal $\rho$ can be achieved by projected gradient descent.

\end{document}